%% file: main.tex
\documentclass{article} 

\usepackage{amsmath}
\usepackage{amssymb}
\usepackage{amsthm}
\usepackage{mathtools}

\usepackage{hyperref}
\hypersetup{
    colorlinks=true,
    linkcolor=blue,
    urlcolor=magenta,
    citecolor=blue, 
    }
\usepackage{xcolor}
\usepackage{times}
\usepackage{fullpage}

\usepackage{natbib}

\newtheorem{claim}{Claim}

\newtheorem{theorem}{Theorem}

\newtheorem{lemma}[theorem]{Lemma}
\newtheorem{remark}[theorem]{Remark}

\newtheorem{innercustomlem}{Restatement of Lemma}
\newenvironment{customlem}[1]
  {\renewcommand\theinnercustomlem{#1}\innercustomlem}
  {\endinnercustomlem}

\newtheorem{innercustomrmk}{Restatement of Remark}
\newenvironment{customrmk}[1]
  {\renewcommand\theinnercustomrmk{#1}\innercustomrmk}
  {\endinnercustomrmk}

\newtheorem{innercustomclm}{Restatement of Claim}
\newenvironment{customclm}[1]
  {\renewcommand\theinnercustomclm{#1}\innercustomclm}
  {\endinnercustomclm}

\newcommand{\rA}{\mathbf{A}}
\newcommand{\rt}{\mathbf{t}}
\newcommand{\ra}{\mathbf{a}}
\newcommand{\rv}{\mathbf{v}}
\newcommand{\rZ}{\mathbf{Z}}
\newcommand{\rN}{\mathbf{N}}
\newcommand{\rX}{\mathbf{X}}
\newcommand{\rY}{\mathbf{Y}}
\newcommand{\rx}{\mathbf{x}}
\newcommand{\ry}{\mathbf{y}}
\newcommand{\rS}{\mathbf{S}}

\newcommand{\rM}{\mathbf{M}}
\newcommand{\rsigma}{\mathbf{\sigma}}

\newcommand{\cH}{\mathcal{S}^{d-1}}
\newcommand{\Rad}{\mathcal{R}}
\newcommand{\Loss}{\mathcal{L}}
\newcommand{\cD}{\mathcal{D}}
\newcommand{\Norm}{\mathcal{N}}
\newcommand{\Disc}{\mathcal{G}}
\newcommand{\ncH}{\mathcal{S}^d}

\newcommand{\Net}{\mathcal{W}}

\newcommand{\AllOne}{\mathbf{1}}
\newcommand{\subH}{\mathcal{H}}

\newcommand{\R}{\mathbb{R}}

\newcommand{\Ball}{\mathbb{B}}
\newcommand{\Balld}{\mathbb{B}_2^d}
\newcommand{\E}{\mathbb{E}}
\newcommand{\Z}{\mathbb{Z}}

\newcommand{\eps}{\varepsilon}
\newcommand{\eqd}{\,{\buildrel d \over =}\,}
\DeclareMathOperator{\support}{supp}
\DeclareMathOperator{\sign}{sign}
\newcommand{\Ints}{\mathbb{Z}}
\newcommand{\finalH}{\mathcal{H}}
\newcommand{\finalX}{\mathcal{X}}
\DeclareMathOperator{\poly}{poly}

\usepackage{bbm}
\usepackage{dsfont}
\newcommand{\indi}[1]{\mathds{1}\{#1\}}

\DeclarePairedDelimiter{\norm}{\parallel}{\parallel}
\DeclarePairedDelimiter{\ipr}{\langle}{\rangle}

\renewcommand{\Pr}{\mathbb{P}}

\title{Tight Generalization Bounds for Large-Margin Halfspaces}

\author{Kasper Green Larsen \qquad\qquad Natascha Schalburg \\ 
        \texttt{\{larsen,n.schalburg\}@cs.au.dk}\\
        Computer Science, Aarhus University
}
\date{}

\begin{document}

\maketitle

\input{abstract}

\input{intro}

\input{overview}

\input{mainproof}

\input{rademacher}

\input{lipschitz}

\input{grid}

\input{factortwo}

\input{acknow}

\bibliography{SVMBib}
\bibliographystyle{abbrvnat}

\appendix

\input{auxiliary}

\end{document}

%% file: abstract.tex
\begin{abstract}
We prove the first generalization bound for large-margin halfspaces that is asymptotically tight in the tradeoff between the margin, the fraction of training points with the given margin, the failure probability and the number of training points. 
\end{abstract}

%% file: intro.tex
\section{Introduction}
Halfspaces are arguably among the simplest and most fundamental classic learning models. Given a normal vector $w \in \R^d$ and a bias $b \in \R$ defining a hyperplane, the corresponding halfspace classifier predicts the label of a data point $x \in \R^d$ by returning $\sign(\langle w, x \rangle+b)$, corresponding to a $+1$ label on points inside the halfspace above the hyperplane, and $-1$ on points below.

Classic examples of learning algorithms for obtaining a halfspace classifier from a training set of points $S = \{(x_i,y_i)\}_{i=1}^n$ with $(x_i, y_i) \in \R^d \times \{-1,1\}$, includes the Perceptron Learning Algorithm (PLA)~(\cite{perceptron}) and Support Vector Machines (SVM)~(\cite{Cortes1995}). A key intuition underlying SVM, is the empirical observation that halfspaces with a large \emph{margin} to the training data tend to generalize well. Ignoring the bias variable $b$ (which may be handled by adding a special feature) and assuming $w \in \cH$ (i.e.\ $w$ has unit length), the margin of the halfspace with normal vector $w$ on a labeled point $(x,y)$ is $y \langle w, x \rangle$. Observe that $\langle w, x\rangle$ gives the signed distance of $x$ from the hyperplane, and the margin is positive when $\sign(\langle w, x \rangle)$ correctly predicts the label $y$. With this definition, \emph{hard-margin} SVM 
computes the normal vector $w$ of the hyperplane with the largest minimum margin.
There are also margin variants of the Perceptron~(\cite{marginperceptron}) that computes a halfspace with minimum margin approaching the optimal, as in hard-margin SVM.

To handle data that is not linearly separable, and to add robustness to outliers, the \emph{soft-margin} SVM relaxes the optimization problem to the following
\begin{align*}
    \min_{w, \xi} \|w\|_2^2 + \lambda \sum_i \xi_i,\ \qquad
    \textrm{s.t. } y_i \langle w, x_i \rangle \geq 1- \xi_i,\qquad
    \xi_i \geq 0.
\end{align*}
Here $\lambda>0$ is a regularization parameter. The soft-margin SVM thus allows for smaller margins on some training points at the cost of a penalty $\lambda \xi_i$.

To theoretically justify and explain the empirical success of focusing on large margins,~\cite{Bartlett98generalizationperformance} proved the first generalization bounds upper bounding the probability $\Loss_\cD(w) := \Pr_{(\rx,\ry) \sim \cD}[\sign(\langle w, \rx \rangle) \neq \ry]$ of misclassifying the label of a new data point. Concretely they first studied the hard-margin case and proved that for any distribution $\cD$ over $\Balld \times \{-1,1\}$ and any $0 < \delta < 1$, it holds with probability at least $1-\delta$ over a training set $\rS \sim \cD^n$ that for every $w \in \cH$ and every margin $0 < \gamma < 1$, if $y\langle w,x\rangle \geq \gamma$ for all $(x,y) \in \rS$ then
\begin{align}
    \Loss_\cD(w)  \leq c \cdot \left(\frac{\ln^2(n)}{\gamma^2 n} + \frac{\ln(e/\delta)}{n}\right),\label{eq:bartlett}
\end{align}
for a constant $c>0$. The restriction to $x \in \Balld$ can be relaxed by multiplying the first term by $R^2$ for $x \in R\cdot \Balld$. A dependency on the scaling of input points is inevitable as margins scale with $\|x\|_2$. Throughout the paper, we state bounds for $R=1$ and remark that all bounds generalize to arbitrary $R$ by replacing $\gamma$ by $\gamma/R$.

Defining $\Loss^\gamma_S(w)$ as the fraction of data points in a training set $S$ where $w$ has margin at most $\gamma$,~\cite{Bartlett98generalizationperformance} also prove a more general result, saying that with probability $1-\delta$ over $\rS \sim \cD^n$, it holds for every $w \in \cH$ that
\begin{align}
    \Loss_\cD(w) \leq \Loss_{\rS}^\gamma(w) + c \cdot \sqrt{\frac{\ln^2(n)}{\gamma^2 n} + \frac{\ln(e/\delta)}{n}}.\label{eq:bartlettSoft}
\end{align}
This was later improved by~\cite{DBLP:journals/jmlr/BartlettM02} using Rademacher complexity arguments, replacing the $\ln^2(n)$ term in~\eqref{eq:bartlettSoft} by $1$. Here, and throughout the paper, we refer to $\Loss^\gamma_\rS(w)$ as the (empirical) \emph{margin loss}.

\paragraph{First-Order Bounds.} 
The first work to interpolate between the hard-margin and soft-margin bounds was due to ~\cite{DBLP:conf/colt/McAllester03}, who gave a general tradeoff of
\begin{align}
    \Loss_\cD(w) \leq \Loss_{\rS}^\gamma(w) + c \cdot \left(\sqrt{ \Loss_\rS^\gamma(w) \cdot \frac{\ln n}{\gamma^2 n} }  + \frac{\ln n}{\gamma^2 n} + \sqrt{\frac{\ln n + \ln(e/\delta)}{n}}\right).\label{eq:mcallester}
\end{align}
Notice how the $\Loss_\rS^\gamma(w)$ term is multiplied onto $\ln n/(\gamma^2 n)$ inside the first square-root. Since the hard-margin case corresponds to this term being $0$, this gives a way of interpolating between the cases. Such bounds are often refered to as \emph{first-order bounds}. Unfortunately,~\eqref{eq:mcallester} still has the seemingly superfluous $\sqrt{(\ln n + \ln(e/\delta))/n}$ term even when $\Loss_\rS^\gamma(w)=0$ and thus falls short of even matching~\eqref{eq:bartlett} in the hard-margin case.

The current state-of-the-art generalization bound is due to~\cite{SVMbest} and states that with probability $1-\delta$ over $\rS \sim \cD^n$, it holds for every $w \in \cH$ that
\begin{align}
    \Loss_\cD(w) \leq \Loss_{\rS}^\gamma(w) + c \cdot \left(\sqrt{\Loss_{\rS}^\gamma(w) \cdot \left(\frac{\ln n}{\gamma^2 n} + \frac{\ln(e/\delta)}{n}\right)} + \frac{\ln n}{\gamma^2 n} + \frac{\ln(e/\delta)}{n}\right).\label{eq:sota}
\end{align}
This improves previous hard-margin bounds by a logarithmic factor and gives a cleaner interpolation between the hard- and soft-margin cases. Furthermore, the bound is close to optimal. Concretely, the dependency on $\delta$ is optimal by tweaking standard results for agnostic PAC learning, see e.g.~\cite{DevroyeGyorfiLugosi1996} [Chapter 11]. Moreover,~\cite{SVMbest} complemented their upper bound by the following lower bound 
\begin{theorem}[\cite{SVMbest}]
\label{thm:lower}
There is a constant $c>0$ such that for any $c n^{-1/2} < \gamma < c^{-1}$, any parameter $0 \leq \tau \leq 1$, and any $n \geq c$, there is a distribution $\cD$ such that it holds with constant probability over $\rS \sim \cD^n$ that there is a $w \in \cH$ such that $\Loss_\rS^\gamma(w) \leq \tau$ and
\begin{align*}
\Loss_\cD(w) &\geq \Loss_\rS^\gamma(w) + c \cdot \left(\sqrt{\tau \cdot \frac{\ln(e/\tau)}{\gamma^2 n}} + \frac{\ln(\gamma^2 n)}{\gamma^2 n} \right) 
\geq \Loss_\rS^\gamma(w) + c \cdot \left(\sqrt{\Loss_{\rS}^\gamma(w) \cdot \frac{\ln(e/\Loss_{\rS}^\gamma(w))}{\gamma^2 n}} + \frac{\ln(\gamma^2 n)}{\gamma^2 n} \right).
\end{align*}
\end{theorem}
Notice how the parameter $\tau$ allows for showing that the upper bound~\eqref{eq:sota} is nearly tight across the range of $\Loss_\rS^\gamma(w)$. Let us also remark that~\cite{SVMbest} states their lower bound with a $\ln n$ rather than $\ln(e \gamma^2 n)$, but require that $\gamma > n^{-0.499}$. A careful examination of their proof however reveals the more general lower bound stated here.

Unfortunately there still remains a discrepancy between the lower bound and~\eqref{eq:sota}. Concretely there is a gap of $\sqrt{\ln n/\ln(e/\Loss_{\rS}^\gamma(w))}$. Moreover, for constant $\Loss_{\rS}^\gamma(w)$, the Rademacher complexity based bound in~\eqref{eq:bartlettSoft} improves over both of the first-order bounds~\eqref{eq:mcallester} and~\eqref{eq:sota}, and matches the lower bound in Theorem~\ref{thm:lower}. This seem to suggest that a better upper bound might be possible.

\paragraph{Our Contribution.}
In this work, we settle the generalization performance of large-margin halfspaces by proving a new upper bound matching the lower bound in Theorem~\ref{thm:lower} across the entire tradeoff between $\gamma$, $\Loss_\rS^\gamma(w)$ and $n$ (and is also tight in terms of $\delta$). Our result is stated in the following theorem
\begin{theorem}
\label{thm:main}
There is a constant $c>0$ such that for any distribution $\cD$ over $\Ball_2^d \times \{-1,1\}$, it holds with probability at least $1-\delta$ over $\rS \sim \cD^n$ that for every $w \in \cH$ and every margin $n^{-1/2} \leq \gamma \leq 1$, we have
\[
\Loss_\cD(w) \leq \Loss_{\rS}^\gamma(w) + c \left(\sqrt{\Loss^{\gamma}_\rS(w) \cdot \left( \frac{ \ln(e/\Loss^{\gamma}_\rS(w))}{\gamma^2 n} + \frac{\ln(e/\delta)}{n} \right)} + \frac{\ln(e\gamma^2 n)}{\gamma^2 n} + \frac{\ln(e/\delta)}{n}  \right).
\]
\end{theorem}
While one might argue that our improvement is small in magnitude, this finally pins down the exact generalization performance of a classic learning model. Furthermore, our proof of Theorem~\ref{thm:main} brings several novel ideas that we hope may find further applications in generalization bounds.

We next proceed to give an overview of our proof and new ideas in Section~\ref{sec:overview}, before giving the full details of the proof in Section~\ref{sec:mainproof}.

%% file: overview.tex
\section{Proof Overview}
\label{sec:overview}
In this section, we present the main ideas in our proof of Theorem~\ref{thm:main}. As our proof builds on, and greatly extends, the work of~\cite{SVMbest} establishing the previous state-of-the-art in~\eqref{eq:sota}, we first present their overall proof strategy and the barriers we need to overcome to obtain our tight generalization bound. Throughout this proof overview, we use the notation $x \lesssim y$ to denote that there is an absolute constant $c>0$ so that $x \leq cy$.

\subsection{Previous Proof}
The proof of~\cite{SVMbest} follows a framework proposed by~\cite{SFBL98} for proving generalization of large-margin \emph{voting classifiers} (i.e.\ boosting). The main idea is to randomly discretize the infinite hypothesis set $\cH$ to obtain a finite set $\Disc \subseteq \R^d \to \{-1,1\}$. If $\Disc$ is small enough, then a standard union bound over all $h \in \Disc$ suffices to bound the difference between the empirical error and the true error $\Loss_\cD(h)$ for every $h \in \Disc$. The key trick is to exploit large margins to allow for a discretization to a smaller $\Disc$.

To elaborate on the above, let us first generalize our notation $\Loss_\cD(w)$ and $\Loss^\gamma_S(w)$ a bit. For a distribution $\cD$ over $\Balld \times \{-1,1\}$, let
$
\Loss^\gamma_\cD(w) := \Pr_{(\rx,\ry) \sim \cD}[\ry\langle w , \rx \rangle \leq \gamma]
$,
that is, $\Loss^\gamma_\cD(w)$ is the probability over a fresh sample $(\rx, \ry)$ from $\cD$, of $w$ having margin no more than $\gamma$ on $(\rx,\ry)$. For a training set $S$, we slightly abuse notation and write $(\rx,\ry) \sim S$ to denote a uniform random sample from $S$. We thus have
\[
\Loss^\gamma_S(w) := \Pr_{(\rx,\ry) \sim S}[\ry\langle w , \rx \rangle \leq \gamma] = \frac{|\{(x,y) \in S : y\langle w , x \rangle \leq \gamma\}|}{|S|}.
\]
When writing $\Loss_{\cD}(w)$ we implicitly mean $\Loss_\cD^0(w)$ and note that this coincides with our previous definition of $\Loss_\cD(w) = \Pr_{(\rx ,\ry)\sim \cD}[\sign(\langle w, \rx \rangle) \neq \ry]$ (defining $\sign(0)=0$).

\paragraph{Random Discretization.}
With this notation, the main idea in the proof of~\cite{SVMbest}, is to apply a Johnson-Lindenstrauss transform~(\cite{JL84}), followed by a random snapping to a grid, in order to map each $w \in \cH$ to a point on a grid $\Disc$ of size $\exp(c k)$ in $\R^k$, with $c > 0$ a sufficiently large constant. In more detail, let $\rA$ be a $k \times d$ matrix with i.i.d.\ $\Norm(0,1/k)$ normal distributed entries. Such a matrix is a classic implementation of the Johnson-Lindenstrauss transform and has the property that $|\langle \rA w, \rA x \rangle - \langle w, x \rangle|$ is greater than $\eps$ with probability at most $\exp(-\eps^2 k/c)$ when $\|w\|_2,\|x\|_2 \leq 1$~(\cite{DG03}). Note that this also preserves the norm of a vector $w$ by considering $x=w$ and noting $\langle w, w \rangle = \|w\|_2^2$. Secondly, following an idea of~\cite{AK17} in a lower bound proof for the Johnson-Lindenstrauss transform,~\cite{SVMbest} randomly round $\rA w$ to a point $h_{\rA,\rt}(w)$ with coordinates integer multiples of $k^{-1/2}$ while guaranteeing that $|\langle h_{\rA,\rt}(w), \rA x \rangle-\langle \rA w, \rA x \rangle|$ is less than $\eps$, except with probability $\exp(- \eps^2 k/c)$. Here we use $\rt$ to denote the randomness involved in the rounding.

Now choosing $\eps= \gamma/4$ gives, by the triangle inequality, that $|\langle h_{\rA,\rt}(w), \rA x \rangle-\langle w, x \rangle| \leq \gamma/2$, except with probability $2\exp(- \gamma^2 k/(16c))$. Furthermore, by plugging in $x=w$ and setting $\eps=1$, we can also deduce that $\|h_{\rA,\rt}(w)\|_2 \leq 2$ except with probability $\exp(-k/c)$. Simple counting arguments show that there are only $\exp(c k)$ many vectors of norm at most $2$ with all coordinates integer multiples of $k^{-1/2}$. That is, except with probability $\exp(-k/c)$, $h_{\rA,\rt}(w)$ belongs to a finite set $\Disc$ of $\exp(c k)$ many points. 

\paragraph{Framework.}
With the above random discretization, the proof of~\cite{SVMbest} now follows the framework of~\cite{SFBL98} by relating $\Loss_\cD(w)$ to $\Loss^{\gamma/2}_{\rA\cD}(h_{\rA,\rt}(w))$ and $\Loss^\gamma_{\rS}(w)$ to $\Loss^{\gamma/2}_{\rA \rS}(h_{\rA,\rt}(w))$. Here $\rA \cD$ is the distribution obtained by sampling $(\rx,\ry) \sim \cD$ and returning $(\rA \rx, \ry)$. Similarly, $\rA S$ is the training set obtained by replacing each $(x,y) \in S$ by $(\rA x, y)$. The intuition is that the random discretization changes margins by no more than $\gamma/2$ for most data points and hence points with margin at most $0$ under $\cD$ often have margin at most $\gamma/2$ under $\rA \cD$ and similarly for $S$ and $\rA S$. Let us make this more formal. We have for any $A,t$ in the support of $\rA,\rt$ that
\begin{align}
\Loss_\cD(w) &\leq \Loss^{\gamma/2}_{A \cD}(h_{A,t}(w)) + \Pr_{(\rx,\ry) \sim \cD}[\ry \langle w, \rx\rangle \leq 0 \wedge \ry \langle h_{A,t}(w) , A\rx\rangle > \gamma/2] \label{eq:Dupper}
\end{align}
Similarly, we have
\begin{align}
\Loss^\gamma_S(w) &\geq \Loss^{\gamma/2}_{A S}(h_{A,t}(w)) - \Pr_{(\rx,\ry) \sim S}[\ry \langle w, \rx\rangle > \gamma \wedge \ry \langle h_{A,t}(w) , A\rx\rangle \leq \gamma/2].\label{eq:Supper}
\end{align}
Taking expectation we see that
\begin{align}
    \Loss_\cD(w) - \Loss_\rS^\gamma(w) &= \E_{\rA,\rt}[\Loss_\cD(w) - \Loss_\rS^\gamma(w)] \nonumber\\
    &\leq \E_{\rA,\rt}[\Loss^{\gamma/2}_{\rA \cD}(h_{\rA,\rt}(w)) -\Loss^{\gamma/2}_{\rA S}(h_{\rA,\rt}(w)) ] \label{eq:needsup}\\
    &+ \E_{\rA,\rt}[\Pr_{(\rx,\ry) \sim \cD}[\ry \langle w, \rx\rangle \leq 0 \wedge \ry \langle h_{\rA,\rt}(w) , \rA \rx\rangle > \gamma/2]]\label{eq:dRel}\\
    &+ \E_{\rA,\rt}[\Pr_{(\rx,\ry) \sim S}[\ry \langle w, \rx\rangle > \gamma \wedge \ry \langle h_{\rA,\rt}(w) , \rA \rx\rangle \leq \gamma/2]]\label{eq:sRel}.
\end{align}
To bound~\eqref{eq:needsup}, we exploit that $h_{\rA,\rt}(w)$ belongs to the grid $\Disc$, except with probability $\exp(- k/c)$. Using Bernstein's inequality (and a careful partitioning of hypotheses $w$ depending on $\Loss_{\cD}^\gamma(w)$), it is possible to union bound over the entire grid and conclude
\begin{align}
\E_{\rA,\rt}[\Loss^{\gamma/2}_{\rA \cD}(h_{\rA,\rt}(w)) -\Loss^{\gamma/2}_{\rA S}(h_{\rA,\rt}(w)) ] &\leq \nonumber\\
\E_{\rA,\rt}[\sup_{h \in \Disc} \Loss^{\gamma/2}_{\rA \cD}(h) -\Loss^{\gamma/2}_{\rA S}(h) ]  + \Pr_{\rA,\rt}[h_{\rA,\rt}(w) \notin \Disc]&\lesssim \nonumber\\
\sqrt{\Loss_{\rS}^\gamma(w) \cdot \frac{\ln(|\finalH_\rA|/\delta)}{n}}  + \frac{\ln(|\finalH_\rA|/\delta)}{n} + \exp(-k/c) &\lesssim \nonumber\\
\sqrt{\Loss_{\rS}^\gamma(w) \cdot \frac{k + \ln(e/\delta)}{n}} + \frac{k + \ln(e/\delta)}{n} + \exp(-k/c).\label{eq:boundgamma/2}
\end{align}

To bound~\eqref{eq:sRel}, we use the guarantees of the random discretization to conclude that
\begin{align*}
    \E_{\rA,\rt}[\Pr_{(\rx,\ry) \sim S}[\ry \langle w, \rx\rangle > \gamma \wedge \ry \langle h_{\rA,\rt}(w) , \rA \rx\rangle \leq \gamma/2]] &= \\
   \E_{(\rx,\ry) \sim S}[ \Pr_{\rA,\rt}[\ry \langle w, \rx\rangle > \gamma \wedge \ry \langle h_{\rA,\rt}(w) , \rA \rx\rangle \leq \gamma/2]] &\leq \\
   \E_{(\rx,\ry) \sim S}[ \Pr_{\rA,\rt}[ \ry \langle h_{\rA,\rt}(w) , \rA \rx\rangle \leq \gamma/2 \mid \ry \langle w, \rx\rangle > \gamma]] &\leq \\
   2\exp(-\gamma^2 k/(16c)).
\end{align*}
We can bound~\eqref{eq:dRel} in a similar fashion (even with slightly better guarantees scaled by $\Loss_\cD(w)$, but this does not help for~\eqref{eq:sRel}). The final generalization error thus becomes
\begin{align}
    \Loss_\cD(w) \leq \Loss_{\rS}^\gamma(w) + c' \cdot \left(\sqrt{\Loss_{\rS}^\gamma(w) \cdot \frac{k + \ln(e/\delta)}{n}} + \frac{k + \ln(e/\delta)}{n} + \exp(-\gamma^2 k/c') \right),\label{eq:finalOld}
\end{align}
where $c'>0$ is a sufficiently large constant. Comparing this expression with the desired bound from Theorem~\ref{thm:main}, we see that we have to choose $k$ large enough that $c'\exp(- \gamma^2 k/c')$ is no larger than
\[
\sqrt{\Loss^{\gamma}_\rS(w) \cdot \left( \frac{ \ln(e/\Loss^{\gamma}_\rS(w))}{\gamma^2 n} + \frac{\ln(e/\delta)}{n} \right)} + \frac{\ln(e\gamma^2 n)}{\gamma^2 n} + \frac{\ln(e/\delta)}{n}.
\]
This basically solves to 
\[
k \gtrsim  \gamma^{-2} \ln\left( \frac{\gamma^2 n}{\Loss_{\rS}^\gamma(w) \ln(e/\Loss_{\rS}^\gamma(w))} \right) \geq \gamma^{-2} \ln\left(\gamma^2 n\right).
\]
Inserting this $k$ in~\eqref{eq:finalOld} recovers the bound by~\cite{SVMbest} stated in~\eqref{eq:sota}.

\paragraph{Barriers.}
In light of the above discussion, we identify some key barriers for the previous proof technique. Concretely, if we examine~\eqref{eq:finalOld}, the term $\sqrt{\Loss_\rS^\gamma(w) k/n}$ requires us to choose $k$ no larger than $c \gamma^{-2} \ln(e/\Loss_\rS^\gamma(w))$ to match the optimal bound we get in Theorem~\ref{thm:main}. Unfortunately, the additive $\exp(-\gamma^2 k/c')$ term originating from handling~\eqref{eq:dRel} and~\eqref{eq:sRel} then becomes $\poly(\Loss_\rS^\gamma(w))$, which is too expensive. In fact, even the additive $\exp(-k/c)$ term from handling~\eqref{eq:needsup} is too expensive for e.g.\ constant $\gamma$. Nonetheless, we will in fact choose such $k$ and identify a tighter strategy for analysing $\Loss_\cD(w) - \Loss^\gamma_{\rS}(w)$.

\subsection{Our Key Improvements}
Our first main observation is that the two upper bounds in~\eqref{eq:Dupper} and~\eqref{eq:Supper} are not completely tight, i.e.\ they are inequalities, not equalities. In~\eqref{eq:Dupper} we for instance ignore points $(x,y)$ that had a margin greater than $0$ for $w$, but where the margin of $(Ax,y)$ is less than $\gamma/2$ for $h_{A,t}(w)$. Taking these into accounts, we get the tighter bounds
\begin{align*}
    \Loss_{\cD}(w) &= \Loss^{\gamma/2}_{A \cD}(h_{A,t}(w))+ \Pr_{(\rx,\ry) \sim \cD}[\ry \langle w, \rx\rangle \leq 0 \wedge \ry \langle h_{A,t}(w) , A\rx\rangle > \gamma/2] \\
    &- \Pr_{(\rx,\ry) \sim \cD}[\ry \langle w, \rx\rangle > 0 \wedge \ry \langle h_{A,t}(w) , A\rx\rangle \leq \gamma/2],
\end{align*}
and
\begin{align*}
\Loss^\gamma_S(w) &= \Loss^{\gamma/2}_{A S}(h_{A,t}(w)) - \Pr_{(\rx,\ry) \sim S}[\ry \langle w, \rx\rangle > \gamma \wedge \ry \langle h_{A,t}(w) , A\rx\rangle \leq \gamma/2] \\
&+ \Pr_{(\rx,\ry) \sim S}[\ry \langle w, \rx\rangle \leq \gamma \wedge \ry \langle h_{A,t}(w) , A\rx\rangle > \gamma/2].
\end{align*}
With these refined bounds, we can now split $\Loss_\cD(w)-\Loss^\gamma_S(w)$ into a sum of three terms:
\begin{align}
    &\Loss^{\gamma/2}_{A \cD}(h_{A,t}(w)) - \Loss^{\gamma/2}_{A S}(h_{A,t}(w)) \nonumber\\ &+
    \Pr_\cD[\ry \langle w, \rx\rangle \leq 0 \wedge \ry \langle h_{A,t}(w) , A\rx\rangle > \gamma/2] - \Pr_S[\ry \langle w, \rx\rangle \leq \gamma \wedge \ry \langle h_{A,t}(w) , A\rx\rangle > \gamma/2]\label{eq:diffDS} \\ &+
    \Pr_S[\ry \langle w, \rx\rangle > \gamma \wedge \ry \langle h_{A,t}(w) , A\rx\rangle \leq \gamma/2] - \Pr_\cD[\ry \langle w, \rx\rangle > 0 \wedge \ry \langle h_{A,t}(w) , A\rx\rangle \leq \gamma/2].\label{eq:diffSD} 
\end{align}
The first line is the same as~\eqref{eq:needsup} from before, but~\eqref{eq:diffDS} and~\eqref{eq:diffSD} improves over~\eqref{eq:dRel} and~\eqref{eq:sRel} by subtracting off a term. Intuitively, our more refined bounds allow us to argue that if the randomized rounding creates a big difference between $\Loss_\cD(w)$ and $\Loss^{\gamma/2}_\cD(h_{A,t}(w))$, then it creates a comparably large difference between $\Loss_{S}^\gamma(w)$ and $\Loss^{\gamma/2}_S(h_{A,t}(w))$, thereby canceling out. We will carefully exploit this in the following. Let us focus on~\eqref{eq:diffDS} and remark that~\eqref{eq:diffSD} is handled symmetrically. For~\eqref{eq:diffDS}, we see that
\begin{align*}
    \Pr_{(\rx,\ry) \sim S}[\ry \langle w, \rx\rangle \leq \gamma \wedge \ry \langle h_{A,t}(w) , A\rx\rangle > \gamma/2] &\geq 
    \Pr_{(\rx,\ry) \sim S}[\ry \langle w, \rx\rangle \leq 0 \wedge \ry \langle h_{A,t}(w) , A\rx\rangle > \gamma/2],
\end{align*}
and thus~\eqref{eq:diffDS} is at most
\begin{align*}
\Pr_{(\rx,\ry) \sim \cD}[\ry \langle w, \rx\rangle \leq 0 \wedge \ry \langle h_{A,t}(w) , A\rx\rangle > \gamma/2] - \Pr_{(\rx,\ry) \sim S}[\ry \langle w, \rx\rangle \leq 0 \wedge \ry \langle h_{A,t}(w) , A\rx\rangle > \gamma/2].
\end{align*}
Now introducing the expectation over the randomized rounding $\rA$ and $\rt$ as in the previous proof, and using linearity of expectation, we want to bound the following expression with probability $1-\delta$ over $\rS \sim \cD^n$
\begin{align}
    \sup_{w \in \cH} \bigg(\E_{\rA,\rt}[\Pr_{(\rx,\ry) \sim \cD}[\ry \langle w, \rx\rangle \leq 0 \wedge \ry \langle h_{\rA,\rt}(w) , \rA\rx\rangle > \gamma/2]] 
    &-\nonumber\\ \E_{\rA,\rt}[\Pr_{(\rx,\ry) \sim S}[\ry \langle w, \rx\rangle \leq 0 \wedge \ry \langle h_{\rA,\rt}(w) , \rA\rx\rangle > \gamma/2]]\bigg) &= \nonumber\\
    \sup_{w \in \cH} \bigg(\E_{(\rx,\ry) \sim \cD}[\Pr_{\rA,\rt}[\ry \langle w, \rx\rangle \leq 0 \wedge \ry \langle h_{\rA,\rt}(w) , \rA\rx\rangle > \gamma/2]] 
    &-\nonumber\\ \E_{(\rx,\ry) \sim \rS}[\Pr_{\rA,\rt}[\ry \langle w, \rx\rangle \leq 0 \wedge \ry \langle h_{\rA,\rt}(w) , \rA\rx\rangle > \gamma/2]]\bigg). \label{eq:needRad}
\end{align}
This now has a form that looks familiar. Concretely, we have a function 
\begin{align}
\psi_w(x,y) = \indi{ y\langle w ,x \rangle \leq 0} \cdot \Pr_{\rA,\rt}[y \langle h_{\rA,\rt}(w),\rA x\rangle > \gamma/2] \label{eq:psi}
\end{align}
for each $w \in \cH$, and wish to bound $\sup_w \E_{(\rx,\ry) \sim \cD}[\psi_w(\rx,\ry)] - \E_{(\rx,\ry) \sim \rS}[\psi_w(\rx,\ry)]$ with high probability over $\rS \sim \cD^n$. Rademacher complexity (see e.g.~\cite{rademacherbound}) is one key tool for bounding such differences. In particular, the contraction principle from~\cite{ledoux1991probability} allows us to bound such a supremum when the functions $\psi_w$ are composite functions $\psi_w = f \circ g_w$ with $f : \R \to \R$ having bounded Lipschitz constant. Indeed margin generalization bounds for halfspaces have previously been proved via Rademacher complexity~(\cite{DBLP:journals/jmlr/BartlettM02}) by composing the two functions $g_w(x,y) = y \langle w, x\rangle$ and $f(u) = \min\{1, \max\{0,\frac{\gamma-u}{\gamma}\}\}$ being the ramp loss. Since the ramp loss $f$ is $\gamma^{-1}$-Lipschitz, one can thus reduce bounding the Rademacher complexity of $\psi_w$ to bounding the Rademacher complexity of the simpler $g_w(x,y) = y\langle w ,x\rangle$. When using this $g_w(x,y)$ and an $L$-Lipschitz $f$, the resulting bound on~\eqref{eq:needRad} becomes $c L/\sqrt{n}$, i.e.\ $\sqrt{1/(\gamma^2 n)}$ when using the ramp loss.

We wish to take a similar approach for our $\psi_w$ in~\eqref{eq:psi}. We thus want to argue that $\psi_w(x,y)$ can be written as a composite function $f(y\langle w, x\rangle)$ of the margin of $(x,y)$ on $w$. Examining~\eqref{eq:psi}, we thus need to argue that the probability over $\rA,\rt$ is a function solely of the original margin $y \langle w, x\rangle$. The matrix $\rA$ was $k \times d$ with i.i.d.\ $\Norm(0,1/k)$ entries. Now if we make the assumption that $x$ is unit length (and not just $\|x\|_2 \leq 1$), then by the rotational invariance of Gaussians, the joint distribution of $\rA w, y\rA x$ can be shown to be completely determined from $y \langle w, x\rangle$. Since the random rounding considers only the vector $\rA w$ (and not the original $w$), it follows that the joint distribution of $h_{\rA,\rt}(w), y \rA x$ is also completely determined from $y \langle w, x\rangle$. This implies that $\psi_w(x,y)$ indeed can be written as a function $f$ of $y \langle w, x\rangle$ alone. 

We thus proceed to bound the Lipschitz constant of the function $f$ in~\eqref{eq:psi}. To avoid discontinuities, we have to alter $\psi_w(x,y)$ somewhat to not include the discontinuous indicator function (similarly to using the ramp loss in previous works), and we eventually bound the Lipschitz constant $L$ by roughly
\[
L \lesssim \gamma^{-1} \Pr_{\rA,\rt}[y \langle h_{\rA,\rt}(w),\rA x\rangle > \gamma/2 \mid y \langle w, x\rangle = 0].
\]
With a slight abuse of notation, we write $\Pr_{\rA,\rt}[y \langle h_{\rA,\rt}(w),\rA x\rangle > \gamma/2 \mid y \langle w, x\rangle = 0]$ to denote the probability $\Pr_{\rA,\rt}[y \langle h_{\rA,\rt}(w),\rA x\rangle > \gamma/2]$ for an arbitrary $x,w \in \cH$ and $y \in \{-1,1\}$ with $y\langle w,x\rangle = 0$ as $y\ipr{w,x}$ completely determines this probability as argued above.

Since our randomized rounding preserves inner products to within $\gamma/2$ except with probability $\exp(-\gamma^2 k/c)$, we get $L \lesssim \gamma^{-1} \exp(-\gamma^2 k/c)$. This finally bounds~\eqref{eq:needRad} by
\[
c \cdot \sqrt{\frac{\exp(-\gamma^2 k/c)}{\gamma^2 n}}.
\]
This should be compared to proof by~\cite{SVMbest} that got a bound of $c\exp(-\gamma^2 k/c)$. This improvement is precisely enough to derive our tight Theorem~\ref{thm:main}. Indeed, as mentioned in~\eqref{eq:boundgamma/2}, we can bound $\Loss^{\gamma/2}_{\rA \cD}(h_{\rA,\rt}(w)) - \Loss^{\gamma/2}_{\rA \rS}(h_{\rA,\rt}(w))$ by 
\begin{align*}
\sqrt{\Loss_{\rS}^\gamma(w) \cdot \frac{k + \ln(e/\delta)}{n}} + \frac{k + \ln(e/\delta)}{n} + \exp(-k/c).
\end{align*}
If we ignore the $\exp(-k/c)$ term and set $k = c' \gamma^{-2} \ln(e/\Loss_{\rS}^\gamma(w))$, this gives the tight bound in Theorem~\ref{thm:main}.

Unfortunately, we cannot afford to ignore the $\exp(-k/c)$ term and we need additional ideas for dealing with it. Recall that in the previous proof by~\cite{SVMbest}, it originates from bounding
\begin{align*}
\E_{\rA,\rt}[\Loss^{\gamma/2}_{\rA \cD}(h_{\rA,\rt}(w)) -\Loss^{\gamma/2}_{\rA S}(h_{\rA,\rt}(w)) ] &\leq 
\E_{\rA,\rt}[\sup_{h \in \Disc} \Loss^{\gamma/2}_{\rA \cD}(h) -\Loss^{\gamma/2}_{\rA S}(h) ]  + \Pr_{\rA,\rt}[h_{\rA,\rt}(w) \notin \Disc],
\end{align*}
and upper bounding $\Pr_{\rA,\rt}[h_{\rA,\rt}(w) \notin \Disc]$ by $\exp(-k/c)$. Here we instead consider an infinite sequence of discretizations/grids $\Disc_0,\Disc_1,\dots,$ and argue that the random rounding $\rA, \rt$ and training set $\rS$ is simultaneously \emph{good} (for some appropriate definition) for all grids with high probability. Here the grids $\Disc_i$ correspond to increasingly large norms of $h_{\rA,\rt}(w)$, i.e.\ $\Disc_i$ contains all vectors of norm at most $2^{i+1} \Balld$ and all coordinates integer multiples of $k^{-1/2}$. Multiple careful applications of Cauchy-Schwartz, Jensen's inequality and upper bounds on the probability that $h_{\rA,\rt}(w) \notin \Disc_i$ allows us to finally get rid of the $\exp(-k/c)$ factor.

%% file: mainproof.tex
\section{Main Proof}
\label{sec:mainproof}
We now set out to prove Theorem~\ref{thm:main} following the proof outline sketched in Section~\ref{sec:overview}. We start by a series of reductions that allow us to focus on a simpler task of establishing Theorem~\ref{thm:main} only for a small range of $\gamma$ and $\Loss_{\rS}^\gamma(w)$. We describe these reductions in Section~\ref{sec:reduct} and then proceed to the main arguments in Section~\ref{sec:mainargs}.

\input{setup}

\input{randomdisc}

%% file: setup.tex
\subsection{Setup}
\label{sec:reduct}
When eventually bounding the Lipschitz constant, as discussed in Section~\ref{sec:overview}, the task turns out to be simpler if $\|x\|_2=1$ (and not just $\|x\|_2 \leq 1$) for all $x$ in the support of $\cD$ and if $|\langle w,x\rangle|$ is sufficiently smaller than $1$ for all hypotheses $w$ and data points $(x,y)$ in the support of $\cD$. We thus start by reducing to this case. 

Consider the following distribution $\cD'$ obtained by sampling an $(\rx,\ry) \sim \cD$ and replacing $\rx$ by $\rx'=(c_\gamma \rx) \times \{ \sqrt{1-c_\gamma^2\|\rx\|_2^2}\} \in \ncH$ for a sufficiently small constant $0 < c_\gamma < 1$. That is, scale down all coordinates of $\rx$ by $c_\gamma$ and append a $(d+1)$'st coordinate taking the value $\sqrt{1-c_\gamma^2\|\rx\|_2^2}$. Then the norm of the resulting point $\rx'$ is $\sqrt{c_\gamma^2\|\rx\|_2^2 + 1-c_\gamma^2\|\rx\|_2^2} = 1$. Similarly, for any $w \in \cH$, consider instead the hypothesis $w'=w \times \{0\}$. We observe that for any $x,w$, we have that $\langle w',x'\rangle = \langle w, c_\gamma x\rangle = c_\gamma\langle w, x \rangle$ and thus lies in the range $[-c_\gamma,c_\gamma]$ by Cauchy-Schwartz. This also implies that $\sign(\langle w', x'\rangle) = \sign(\langle w, x\rangle)$ and thus the generalization error of $w'$ under $\cD'$ and $w$ under $\cD$ are the same. 

With this in mind, we define $\finalH := \cH \times \{0\}$ and $\finalX$ as the set of all vectors $x$ in $\ncH$ where the norm of $x$ without its $(d+1)$'st coordinate is at most $c_\gamma$.

From hereon, we let $\cD$ be an arbitrary distribution over $\finalX \times \{-1,1\}$, and set out to prove that there is a constant $c>1$, such that with probability at least $1-\delta$ over $\rS \sim \cD^n$, it holds for all margins $\gamma \in (n^{-1/2}, c_\gamma]$ and all $w \in \finalH$ that
\begin{align}
\Loss_\cD(w) \leq \Loss_{\rS}^\gamma(w) + c \left(\sqrt{\Loss^{\gamma}_\rS(w) \cdot \left( \frac{ \ln(e/\Loss^{\gamma}_\rS(w))}{\gamma^2 n} + \frac{\ln(e/\delta)}{n} \right)} + \frac{\ln(e\gamma^2 n)}{\gamma^2 n} + \frac{\ln(e/\delta)}{n}  \right). \label{eq:maingoal}
\end{align}
Note that Theorem~\ref{thm:main} follows as an immediate corollary since margins change by a $c_\gamma$ factor in our transformation of the input distribution. Since $c_\gamma$ is a constant, this disappears in the constant factor $c$ in Theorem~\ref{thm:main} (note that for margins $\gamma \in [n^{-1/2}, c_\gamma^{-1} n^{-1/2})$ in Theorem~\ref{thm:main}, we cannot use the reduction, but here Theorem~\ref{thm:main} follows trivially as $c \ln(e \gamma^2 n)/(\gamma^2 n) > 1$ for sufficiently large $c$). 

\paragraph{Smaller Tasks.}
We now break the task of establishing~\eqref{eq:maingoal} into smaller tasks, where we consider margins $\gamma$ in a small range $(\gamma_i, \gamma_{i+1}]$ and only vectors $w \in \finalH$ with $\Loss^{(3/4)\gamma_i}_\cD(w)$ in a small range $(\ell_j,\ell_{j+1}]$. The purpose here is, that for one sub-task, we can treat margins and margin losses as the same to within constant factors. A union bound over all the sub-tasks then suffices to establish~\eqref{eq:maingoal}.

For a given distribution $\cD$, partition the range of values of the margin $\gamma \in (n^{-1/2},c_\gamma]$ into intervals $\Gamma_i = (2^{i-1} n^{-1/2}, 2^{i} n^{-1/2}]$ for $i=1,\dots,\lg_2(c_\gamma n^{1/2})$. Similarly, partition the possible values of $\Loss^\gamma_\cD(w) \in [0,1]$ into intervals $L_0 = [0, n^{-1}]$ and $L_i = (2^{i-1} n^{-1}, 2^{i} n^{-1}]$ with $i=1,\dots \lg_2 n$.

For a pair $(\Gamma_i, L_j)$ with $\Gamma_i = (\gamma_i, \gamma_{i+1}]$, define
\[
\subH(\Gamma_i,L_j) = \{ w \in \finalH : \Loss^{(3/4)\gamma_{i}}_\cD(w) \in L_j \}.
\]
For each pair $(\Gamma_i,L_j)$ we now prove an equivalent of~\eqref{eq:maingoal}, but tailored to the sub-task. The result is stated in the following lemma
\begin{lemma}
\label{lem:subgoal}
There is a constant $c>1$, such that for any $0 < \delta < 1$ and any pair $(\Gamma_i, L_j) = ((\gamma_i, \gamma_{i+1}], (\ell_j, \ell_{j+1}])$, it holds with probability at least $1-\delta$ over a random sample $\rS \sim \cD^n$ that
\begin{align}
\sup_{\substack{w \in \subH(\Gamma_i, L_j)\\\gamma \in \Gamma_i}} \left|\Loss_\cD(w) - \Loss^{\gamma}_\rS(w) \right| &\leq 
c \left(\sqrt{\ell_{j+1}\left( \frac{ \ln(e/\ell_{j+1})}{\gamma_{i+1}^2 n} + \frac{\ln(e/\delta)}{n} \right)} + \frac{\ln(e/\ell_{j+1})}{\gamma_{i+1}^2 n} + \frac{\ln(e/\delta)}{n} \right)\label{eq:subgoal}
\end{align}
\end{lemma}
Observe that while~\eqref{eq:maingoal} depends on $\gamma$ and ~\eqref{eq:subgoal} depends on $\gamma_{i+1}$, this is fine since $\gamma \leq \gamma_{i+1}$ for all $\gamma \in \Gamma_i$. However, recall that $\subH(\Gamma_i,L_j)$ refers to $w \in \finalH$ with $\Loss_{\cD}^{(3/4)\gamma_i}(w) \in L_j = (\ell_j,\ell_{j+1}]$. But the $\ell_{j+1}$ terms in~\eqref{eq:subgoal} need to be replaced by $\Loss^\gamma_{\rS}(w)$ to obtain~\eqref{eq:maingoal}. Thus we relate the two via the following lemma
\begin{lemma}
\label{lem:subgoal2}
There is a constant $c>1$, such that for any $0 < \delta < 1$ and any $\Gamma_i = (\gamma_i, \gamma_{i+1}]$, it holds with probability at least $1-\delta$ over a random sample $\rS \sim \cD^n$ that
\begin{align}
\forall w \in \finalH : \Loss_{\rS}^{\gamma_i}(w) \geq \frac{\Loss_{ \cD}^{(3/4)\gamma_i}(w)}{4} - c \left(\frac{\ln(e\gamma_{i+1}^2 n)}{\gamma_{i+1}^2 n} - \frac{\ln(e/\delta)}{n}\right).\label{eq:subgoal2}
\end{align}
\end{lemma}
We combine the sub-tasks and conclude
\begin{claim}
\label{clm:union}
    For any $0 < \delta < 1$, it holds with probability $1-\delta$ over $\rS \sim \cD^n$ that~\eqref{eq:subgoal} and~\eqref{eq:subgoal2} simultaneously hold for all $(\Gamma_i,L_j)$ and $\Gamma_i$, with slightly different constants $c$.
\end{claim}
Since Claim~\ref{clm:union} follows by a simple union bound, exploiting that for different values of $\ell_{j+1}$ and $\gamma_{i+1}$, we can afford to use different $\delta_{i,j} \approx \delta \exp(-\gamma_{i+1}^2 \ln(e/\ell_{j+1}))$ and $\delta_i \approx \delta \exp(-\gamma_{i+1}^{-2} \ln(e \gamma^2_{i+1} n))$, we have deferred the proof to Appendix~\ref{sec:aux}.

A simple combination of~\eqref{eq:subgoal} and~\eqref{eq:subgoal2} now gives
\begin{claim}
\label{clm:combine}
For any $0 < \delta < 1$ and training set $S$, if~\eqref{eq:subgoal} and~\eqref{eq:subgoal2} hold simultaneously for all $(\Gamma_i, L_j)$ and $\Gamma_i$, then~\eqref{eq:maingoal} holds for all $\gamma \in (n^{-1/2}, c_\gamma]$ and all $w \in \finalH$ for large enough constant $c>1$ in~\eqref{eq:maingoal}.
\end{claim}
Claim~\ref{clm:combine} follows by using that $\gamma \leq \gamma_{i+1}$ for $\gamma \in \Gamma_i$, and by using Lemma~\ref{lem:subgoal2} to relate all occurrences of $\ell_{j+1}$ in~\eqref{eq:subgoal} to $\Loss^\gamma_S(w)$. As this is rather straight forward calculations, we have deferred the proof to Appendix~\ref{sec:aux}.

What remains is thus to establish Lemma~\ref{lem:subgoal} and Lemma~\ref{lem:subgoal2}, where we may now focus on a small range of $\gamma$ and $\Loss^{(3/4)\gamma_i}_{\cD}(w)$. While both require substantial work and non-trivial arguments, the proof of Lemma~\ref{lem:subgoal2} follows mostly the previous work by~\cite{SVMbest} and has thus been deferred to Section~\ref{sec:withinconstant}.

%% file: randomdisc.tex
\subsection{Random Discretization}
\label{sec:mainargs}
We now set out to prove Lemma~\ref{lem:subgoal}. So let $0 < \delta < 1$, and fix a pair $(\Gamma_i, L_j)$. Following the proof outline in Section~\ref{sec:overview}, we now consider the following random discretization of hypotheses in $\subH(\Gamma_i, L_j)$: Let $k= k(i,j)$ be an integer parameter to be determined. Sample a random $k \times d$ matrix $\rA$ with each entry $\Norm(0,1/k)$ distributed as well as $k$ random offsets $\rt = (\rt_1,\dots,\rt_k)$ all independent and uniformly distributed in $[0,1]$.

Let $\Disc$ be the set of all vectors in $\R^k$ with coordinates in 
\[
\{(1/2)(10 \sqrt{k})^{-1} + z (10 \sqrt{k})^{-1}  \mid z \in \Ints\}.
\]
For $w \in \finalH$ and an outcome $(A,t)$ of $(\rA,\rt)$, define $h_{A,t}(w) \in \Disc$ as the vector obtained as follows: Consider each coordinate $(Aw)_i$ and let $z_i$ denote the integer such that 
\[
(1/2)(10 \sqrt{k})^{-1}  + z_i (10 \sqrt{k})^{-1} \leq (Aw)_i < (1/2)(10 \sqrt{k})^{-1}  + (z_i+1) (10\sqrt{k})^{-1}.
\]
Let $(h_{A,t}(w))_i$ equal $(1/2)(10 \sqrt{k})^{-1} + z_i (10\sqrt{k})^{-1}$ if $t_i \leq p(z_i)$ ($(Aw)_i$ rounded down) and otherwise let it equal $(1/2)(10 \sqrt{k})^{-1}  + (z_i + 1)(10\sqrt{k})^{-1}$. We choose $p(z_i) \in [0,1]$ such that 
\begin{align}
(Aw)_i &= 
p(z_i)\left(\frac{1}{2 \cdot 10 \sqrt{k}} +\frac{z_i}{10\sqrt{k}}\right) + (1-p(z_i))\left(\frac{1}{2 \cdot 10 \sqrt{k}} +\frac{z_i+1}{10\sqrt{k}}\right) \label{eq:expectround}
\end{align}
i.e.\ for fixed $A$, the expected value of the coordinates satisfy $\E_{\rt}[(h_{A,\rt}(w))_i]=(Aw)_i$. 
\begin{remark}
\label{rmk:pIsProb}
The value $p(z_i)$ satisfying~\eqref{eq:expectround} has $p(z_i) \in [0,1]$.
\end{remark}
We thus have that $p(z_i)$ is a well-defined probability. We prove Remark~\ref{rmk:pIsProb} in Appendix~\ref{sec:aux}. The random discretization has the desirable property that it approximately preserves margins/inner products as stated in the following
\begin{lemma}
\label{lem:concdiscretize}
    There is a constant $c>0$, such that for any integer $k \geq 1$, $w \in \finalH, x \in \finalX$ and any $\gamma \in (0,1]$, it holds that 
    $
    \Pr_{\rA,\rt}[|\langle h_{\rA,\rt}(w),\rA x\rangle - \langle w, x\rangle| > \gamma] < c\exp(-\gamma^2 k/c)
    $.
\end{lemma}
The proof of Lemma~\ref{lem:concdiscretize} follows the work by~\cite{AK17} in their work on lower bounds for the Johnson-Lindenstrauss transform, and has thus been deferred to Appendix~\ref{sec:aux}. We now observe that
\begin{align*}
    \Loss_{\cD}(w) &= \Loss^{\gamma_i/2}_{\rA \cD}(h_{\rA,\rt}(w)) + \Pr_{(\rx,\ry) \sim \cD}[y \langle h_{\rA,\rt}(w), \rA \rx\rangle > \gamma_i/2 \wedge \ry \langle w, \rx\rangle \leq 0] \\
    &- \Pr_{(\rx,\ry) \sim \cD}[\ry \langle h_{\rA,\rt}(w), \rA \rx\rangle \leq \gamma_i/2 \wedge \ry \langle w, \rx\rangle > 0].
\end{align*}
Similarly, we have for $\gamma \in \Gamma_i$ and any training set $S$ that
\begin{align*}
    \Loss_{S}^\gamma(w) &= \Loss^{\gamma_i/2}_{\rA S}(h_{\rA,\rt}(w)) + \Pr_{(\rx,\ry) \sim S}[\ry \langle h_{\rA,\rt}(w), \rA \rx\rangle > \gamma_i/2 \wedge \ry \langle w, \rx\rangle \leq \gamma] \\
    &- \Pr_{(\rx,\ry) \sim S}[\ry \langle h_{\rA,\rt}(w), \rA\rx\rangle \leq \gamma_i/2 \wedge \ry \langle w, \rx\rangle > \gamma].
\end{align*}
We now have for any $\gamma \in \Gamma_i$ that
\begin{align}
\sup_{w \in \subH(\Gamma_i,L_j)} \Loss_\cD(w) - \Loss^\gamma_S(w) =& \nonumber \\ 
\sup_{w \in \subH(\Gamma_i,L_j)} \big(\E_{\rA,\rt} [\Loss^{\gamma_i/2}_{\rA \cD}(h_{\rA,\rt}(w)) - \Loss^{\gamma_i/2}_{\rA S}(h_{\rA,\rt}(w))] +& \nonumber\\ \E_{\rA,\rt} [\Pr_{\cD}[\ry \langle h_{\rA,\rt}(w), \rA \rx\rangle > \gamma_i/2 \wedge \ry \langle w, \rx\rangle \leq 0] -\Pr_{S}[\ry \langle h_{\rA,\rt}(w), \rA \rx\rangle > \gamma_i/2 \wedge \ry \langle w, \rx\rangle \leq \gamma] ] +& \nonumber\\
\E_{\rA,\rt} [\Pr_{S}[\ry \langle h_{\rA,\rt}(w), \rA \rx\rangle \leq \gamma_i/2 \wedge \ry \langle w, \rx\rangle > \gamma] -\Pr_{\cD}[\ry \langle h_{\rA,\rt}(w), \rA \rx\rangle \leq \gamma_i/2 \wedge \ry \langle w, \rx\rangle > 0] ]\big).\label{eq:3terms}
\end{align}
A critical observation is that the distribution of $y \langle h_{\rA,\rt}(w),\rA x \rangle$ depends only on $y \langle w, x \rangle$. 
\begin{claim}
\label{clm:distDeterm}
For any $(x,y) \in \finalX \times \{-1,1\}$ and any $w \in \finalH$, the distribution of $y \langle h_{\rA,\rt}(w),\rA x \rangle$ is completely determined from $y \langle w, x \rangle$.
\end{claim}
We prove Claim~\ref{clm:distDeterm} in Section~\ref{sec:lip} by exploiting that the entries of $\rA$ are i.i.d.\ $\Norm(0,1/k)$ distributed and using the rotational invariance of the Gaussian distribution.

As outlined in the proof overview in Section~\ref{sec:overview}, we can now use Claim~\ref{clm:distDeterm} together with the contraction inequality of Rademacher complexity to bound several of the terms in~\eqref{eq:3terms}. Similarly to the introduction of the ramp loss in classic proofs of generalization for large-margin halfspaces, we need to introduce a continuous function upper bounding the probabilities above. With this in mind, we now define the following functions $\phi$ and $\rho$:
\[
\phi(\alpha) = \begin{cases} \Pr_{\rA, \rt}[y \langle h_{\rA,\rt}(w), \rA x\rangle > \gamma_i/2 \mid y \langle w, x \rangle = \alpha] & \text{if } -c_\gamma \leq \alpha \leq 0 \\
                      \frac{(\gamma_i-\alpha)}{\gamma_i}\Pr_{\rA, \rt}[y \langle h_{\rA,\rt}(w), \rA x\rangle > \gamma_i/2 \mid y \langle w, x \rangle = 0]                                    & \text{if } 0 < \alpha \leq \gamma_i      \\
                      0 & \text{if } \gamma_i < \alpha \leq c_\gamma
        \end{cases}
\]
\[
\rho(\alpha) = \begin{cases} \Pr_{\rA, \rt}[y \langle h_{\rA,\rt}(w), \rA x\rangle \leq \gamma_i/2 \mid y \langle w, x \rangle = \alpha] & \text{if } \gamma_i < \alpha \leq c_\gamma\\
                      \frac{\alpha}{\gamma_i}\Pr_{\rA, \rt}[y \langle h_{\rA,\rt}(w), \rA x\rangle \leq \gamma_i/2 \mid y \langle w, x \rangle = \gamma_i]                                    & \text{if } 0 < \alpha \leq \gamma_i      \\
                      0 & \text{if } -c_\gamma \leq \alpha \leq 0
        \end{cases}
\]
Here we slightly abuse notation and write $\Pr_{\rA, \rt}[y \langle h_{\rA,\rt}(w), \rA x\rangle > \gamma_i/2 \mid y \langle w, x \rangle = \alpha]$ to denote the probability $\Pr_{\rA, \rt}[y \langle h_{\rA,\rt}(w), \rA x\rangle > \gamma_i/2]$ for an arbitrary $w \in \finalH, (x,y) \in \finalX \times \{-1,1\}$ with $y \ipr{w,x}=\alpha$ and remark that this probability is the same for all such $w,x,y$ by Claim~\ref{clm:distDeterm}.

We now observe that $\phi$ and $\rho$ upper and lower bounds the terms in~\eqref{eq:3terms}
\begin{remark}
\label{rmk:phirho}
For any training set $S$ and distribution $\cD$ over $\finalX \times \{-1,1\}$, we have
\begin{align*}
    \E_{\rA,\rt} [\Pr_{(\rx,\ry) \sim \cD}[\ry \langle h_{\rA,\rt}(w), \rA \rx\rangle > \gamma_i/2 \wedge \ry \langle w, \rx\rangle \leq 0]] &\leq\E_{(\rx,\ry) \sim \cD}[\phi(\ry \langle w, \rx \rangle)] \\
    \E_{\rA,\rt} [\Pr_{(\rx,\ry) \sim S}[\ry \langle h_{\rA,\rt}(w), \rA \rx\rangle > \gamma_i/2 \wedge \ry \langle w, \rx\rangle \leq \gamma]] &\geq \E_{(\rx,\ry) \sim S}[\phi(\ry \langle w, \rx \rangle)] \\
    \E_{\rA,\rt} [\Pr_{(\rx,\ry) \sim S}[\ry \langle h_{\rA,\rt}(w), \rA \rx\rangle \leq \gamma_i/2 \wedge \ry \langle w, \rx\rangle > \gamma]] &\leq \E_{(\rx, \ry) \sim S}[\rho(\ry \langle w, \rx \rangle)]\\
    \E_{\rA,\rt} [\Pr_{(\rx,\ry) \sim \cD}[\ry \langle h_{\rA,\rt}(w), \rA \rx\rangle \leq \gamma_i/2 \wedge \ry \langle w, \rx\rangle > 0]] &\geq \E_{(\rx, \ry) \sim \cD}[\rho(\ry\langle w, \rx \rangle)].
\end{align*}
\end{remark}
The proof of Remark~\ref{rmk:phirho} follows from the definition of $\phi$ and $\rho$, along with monotonicity of $\Pr_{\rA,\rt}[y \ipr{h_{\rA,\rt}(w),\rA x} > \gamma_i \mid y\ipr{w,x}=\alpha]$ as a function of $\alpha$. The proofs have been deferred to Appendix~\ref{sec:aux}.
Continuing from~\eqref{eq:3terms} using Remark~\ref{rmk:phirho}, linearity of expectation and the triangle inequality, we have for any $\gamma \in \Gamma_i$ that
\begin{align}
\sup_{w \in \subH(\Gamma_i, L_j)} \Loss_\cD(w) - \Loss^\gamma_S(w) \leq&
\sup_{w \in \subH(\Gamma_i, L_j)} \left|\E_{\rA,\rt} [\Loss^{\gamma_i/2}_\cD(h_{\rA,\rt}(w)) - \Loss^{\gamma_i/2}_S(h_{\rA,\rt}(w))]\right| \label{eq:middle}\\
+&\sup_{w \in \subH(\Gamma_i, L_j)} \left| \E_{(\rx,\ry) \sim \cD}[\phi(\ry \langle w, \rx \rangle)] -\E_{(\rx,\ry) \sim S}[\phi(\ry \langle w, \rx \rangle)]\right| \label{eq:phi}\\
+&\sup_{w \in \subH(\Gamma_i, L_j)} \left|\E_{(\rx,\ry) \sim \cD}[\rho(\ry \langle w, \rx \rangle)] -\E_{(\rx,\ry) \sim S}[\rho(\ry \langle w, \rx \rangle)] \right|.\label{eq:rho}
\end{align}
In Section~\ref{sec:meetinmid}, we carefully use Bernstein's plus a (highly non-trivial) union bound over infinitely many grids of increasing size to bound~\eqref{eq:middle} as follows
\begin{lemma}
\label{lem:supW}
There is a constant $c>0$ such that with probability at least $1-\delta$ over $\rS \sim \cD^n$ we have
\begin{align*}
    \eqref{eq:middle} &\leq c \left(\sqrt{\frac{(\ell_{j+1}+ \exp(-\gamma_{i+1}^2k/c)) (k + \ln(e/\delta))}{n}} + \frac{(k + \ln(e/\delta))}{n} \right).
\end{align*}
\end{lemma}
In Section~\ref{sec:rademacher}, we then use Rademacher complexity and a bound on the Lipschitz constants of $\phi$ and $\rho$ to bound~\eqref{eq:phi} and~\eqref{eq:rho} as follows
\begin{lemma}
\label{lem:supPhi}
There are constants $c,c'>0$ such that with probability at least $1-\delta$ over $\rS \sim \cD^n$ we have
\begin{align*}
\max\{\eqref{eq:phi}, \eqref{eq:rho} \}\leq
c\exp(-\gamma_{i+1}^2k/c) \cdot \sqrt{(k + \gamma_{i+1}^{-2} + \ln(e/\delta))/n} .
\end{align*}
provided that $k \geq c' \gamma_{i+1}^{-2}$. 
\end{lemma}
To balance the expressions in Lemma~\ref{lem:supW} and Lemma~\ref{lem:supPhi}, we now set $k = c \gamma_{i+1}^{-2} \ln(e/\ell_{j+1})$ for a sufficiently large constant $c>0$ so that $\exp(-\gamma_{i+1}^2 k/c) \leq \ell_{j+1}/e$ and $k \geq c'\gamma_{i+1}^{-2}$. Combining Lemma~\ref{lem:supW} and Lemma~\ref{lem:supPhi} via a union bound with $\delta'=\delta/2$ and inserting into~\eqref{eq:middle},~\eqref{eq:phi} and~\eqref{eq:rho} gives
\begin{align*}
\sup_{w \in \subH(\Gamma_i, L_j)} \Loss_\cD(w) - \Loss^\gamma_S(w) \leq& \\c\left(\sqrt{\frac{\ell_{j+1}(\gamma_{i+1}^{-2}\ln(e/\ell_{j+1}) + \ln(e/\delta))}{n}}  + \frac{\gamma_{i+1}^{-2} \ln(e/\ell_{j+1})+\ln(e/\delta)}{n}\right) &+\\
c \left( \ell_{j+1} \sqrt{(\gamma_{i+1}^{-2} \ln(e/\ell_{j+1}) + \ln(e/\delta))/n} \right),
\end{align*}
for a constant $c>0$. This completes the proof of Lemma~\ref{lem:subgoal}, which together with Lemma~\ref{lem:subgoal2} completes the proof of our main result, Theorem~\ref{thm:main}.

%% file: rademacher.tex
\section{Rademacher Bounds}
\label{sec:rademacher}
In this section, we use Rademacher complexity and the contraction inequality to prove Lemma~\ref{lem:supPhi}. We focus on bounding~\eqref{eq:phi} and note that~\eqref{eq:rho} is handled symmetrically.

%\begin{customlem}{\ref{lem:supPhi}}
%There are constants $c,c'>0$ such that with probability at %least $1-\delta$ over $\rS \sim \cD^n$ we have
%\begin{align*}
%\sup_{w \in \subH(\Gamma_i,L_j)} \left| \E_{(\rx,\ry) \sim \cD}%[\phi(\ry \langle w, \rx \rangle)] -\E_{(\rx,\ry) \sim \rS}%[\phi(\ry \langle w, \rx \rangle)]\right| &\leq\\
%c\exp(-\gamma_{i+1}^2k/c) \cdot \sqrt{\frac{k + %\gamma_{i+1}^{-2} + \ln(e/\delta)}{n}},
%\end{align*}
%provided that $k \geq c' \gamma_{i+1}^2$. The same bound holds %with $\rho$ in place of $\phi$.
%\end{customlem}
For a training set $S \in (\finalX \times \{-1,1\})^n$, consider the empirical Rademacher complexity (for $\rsigma=(\rsigma_1,\dots,\rsigma_{n})$ a vector of independent and uniform variables in $\{-1,1\}$):
\begin{align*}
\hat{\Rad}_{\phi, \subH(\Gamma_i,L_j)}(S) &= \frac{1}{n} \cdot \E_{\rsigma}\left[ \sup_{w \in \subH(\Gamma_i,L_j)} \sum_{(x_i,y_i) \in S} \rsigma_i \phi(y_i \langle w, x_i \rangle )\right] \\
&\leq \frac{1}{n} \cdot \E_{\rsigma}\left[ \sup_{w \in \finalH} \sum_{(x_i,y_i) \in S} \rsigma_i \phi(y_i \langle w, x_i \rangle )\right]
\end{align*}
If $\phi$ is $L_\phi$-Lipschitz, then the contraction inequality from~\cite{ledoux1991probability} gives that
\[
\hat{\Rad}_{\phi, \finalH}(S) \leq \frac{L_\phi}{n} \cdot \E_{\rsigma}\left[ \sup_{w \in \finalH} \sum_{(x_i,y_i) \in S} \rsigma_i y_i \langle w, x_i \rangle  \right].
\]
Using Cauchy-Schwartz, this is bounded by
\begin{align*}
\hat{\Rad}_{\phi, \finalH}(S) \leq& \frac{L_\phi}{n} \cdot  \E_{\rsigma}\left[ \sup_{w \in \finalH}  \left\langle w, \sum_{(x_i,y_i) \in S}\rsigma_i y_i x_i \right\rangle  \right] \\
\leq& \frac{L_\phi}{n} \cdot \left(\sup_{w \in \finalH} \|w\|_2\right) \cdot \E_{\rsigma}\left[\left\| \sum_{(x_i,y_i) \in S}\rsigma_i y_i x_i\right\|_2 \right] \\
\leq& \frac{L_\phi}{n} \cdot \sqrt{\E_{\rsigma}\left[\left\| \sum_{(x_i,y_i) \in S}\rsigma_i y_i x_i\right\|_2^2 \right]}\\
=& \frac{L_\phi}{\sqrt{n}} \cdot \sqrt{\sum_{(x_i,y_i) \in S} \sum_{(x_j,y_j) \in S} \E_\sigma[\sigma_i \sigma_j] y_i y_j \ipr{x_i, x_j}} \\
=& \frac{L_\phi}{\sqrt{n}}.
\end{align*}
Since this inequality holds for all $S$ with each $(x,y) \in S$ satisfying $\|x\|_2 = 1$, we have for the distribution $\cD$ that the Rademacher complexity
\[
\Rad_{\cD,\phi,\finalH}(n) = \E_{\rS \sim \cD^n}[\hat{\Rad}_{\phi,\finalH}(\rS)]
\]
satisfies $\Rad_{\cD,\phi,\finalH}(n) \leq L_\phi/\sqrt{n}$. By Lemma~\ref{lem:concdiscretize} and $\gamma_i = \gamma_{i+1}/2$, we have that $\phi$ is bounded by 
\[
0 \leq \phi(\alpha) \leq \max_{-c_\gamma \leq \alpha \leq 0}\Pr_{\rA, \rt}[y \langle h_{\rA,\rt}(w), \rA x\rangle > \gamma_i/2 \mid y \langle w, x \rangle = \alpha] \leq c\exp(-k\gamma_{i+1}^2/c),
\]
for a constant $c>0$. We conclude from standard results on Rademacher complexity (see e.g.~\cite{rademacherbound}), that with probability $1-\delta$ over a sample $\rS \sim \cD^n$ it holds that
\begin{align*}
\sup_{w \in \subH(\Gamma_i,L_j)} \left| \E_{(\rx,\ry) \sim \cD}[\phi(\ry \langle w, \rx \rangle)] -\E_{(\rx,\ry) \sim \rS}[\phi(\ry \langle w, \rx \rangle)]\right| \leq& \\
2 \Rad_{\cD,\phi,\finalH}(n) + c_R\left(c\exp(-k\gamma_{i+1}^2/c)\sqrt{\frac{\ln(1/\delta)}{n}} \right) \leq&\\
\frac{2 L_\phi}{\sqrt{n}} + c_R\left(c\exp(-k\gamma_{i+1}^2/c)\sqrt{\frac{\ln(1/\delta)}{n}} \right).
\end{align*}
where $c_R>0$ is a constant. Symmetric arguments bounds $\rho$ by the same, with the Lipschitz constant $L_\rho$ of $\rho$ in place of $L_\phi$.

We now use the following bound on the Lipschitz constants of $\phi$ and $\rho$
\begin{lemma}
\label{lem:finalLip}
There are constants $c_L, c>0$ such that the Lipschitz constants $L_\phi$ and $L_\rho$ of $\phi$ and $\rho$ are bounded by
\[
c_L \exp(-\gamma_{i+1}^2 k/c_L)\cdot \left(\sqrt{k} + \gamma^{-1}_{i+1}\right),
\]
when $k \geq c \gamma_{i+1}^{-2}$.
\end{lemma}
We prove this lemma in the next section. We thus conclude that with probability at least $1-\delta$ over $\rS \sim \cD^n$, we have
\begin{align*}
\sup_{w \in \subH(\Gamma_i,L_j)} \left| \E_{(\rx,\ry) \sim \cD}[\phi(\ry \langle w, \rx \rangle)] -\E_{(\rx,\ry) \sim \rS}[\phi(\ry \langle w, \rx \rangle)]\right| \leq& \\
2 \cdot \frac{c_L \exp(-\gamma_{i+1}^2k/c_L)\left(\sqrt{k} +\gamma^{-1}_{i+1}\right)}{\sqrt{n}} + c_R \left(c \exp(-k \gamma_{i+1}^2/c) \sqrt{\frac{\ln(1/\delta)}{n}} \right).
\end{align*}
The same bound holds for $\rho$ via Lemma~\ref{lem:finalLip}, which completes the proof of Lemma~\ref{lem:supPhi}.

%% file: lipschitz.tex
\subsection{Bounding the Lipschitz Constants}
\label{sec:lip}
In this section, we proceed to bound the Lipschitz constants of $\phi$ and $\rho$ and thereby prove Lemma~\ref{lem:finalLip}.
%\begin{customlem}{\ref{lem:finalLip}}
%There are constants $c_L,c>0$ such that the Lipschitz constants $L_\phi$ and $L_\rho$ of $\phi$ and $\rho$ are bounded by
%\[
%c_L \exp(-\gamma_{i+1}^2 k/c_L)\cdot \left(\sqrt{k} + %\gamma^{-1}_{i+1}\right),
%\]
%when $k \geq c \gamma_{i+1}^{-2}$.
%\end{customlem}
We  split it into two tasks depending on the value of $\alpha$. The simplest case is the following
\begin{lemma}
\label{lem:easyLip}
There is a constant $c>0$ such that the Lipschitz constants of $\phi$ and $\rho$, when $0 < \alpha \leq \gamma_i$, are less than:
\[
\frac{c \exp(-k\gamma_{i+1}^2/c)}{\gamma_{i+1}}.
\]
\end{lemma}

\begin{proof}
Since $\phi$ is linear when $0 < \alpha \leq \gamma_i$, its Lipschitz constant equals the slope of the line, i.e.
\begin{align*}
\frac{1}{\gamma_i} \cdot \Pr_{\rA,\rt}[y \langle h_{\rA,\rt}(w),\rA x\rangle > \gamma_i/2 \mid  y\langle w,x \rangle = 0].
\end{align*}
By Lemma~\ref{lem:concdiscretize} and using $\gamma_i = \gamma_{i+1}/2$, this is bounded by
\begin{align*}
    \frac{c\exp(-\gamma_{i+1}^2 k/c)}{\gamma_{i+1}},
\end{align*}
for a constant $c>0$. The same arguments applies immediately to $\rho$.
\end{proof}
The trickier case is when $\alpha \in [-c_\gamma, 0]$ for $\phi$ and when $\alpha \in (\gamma_i, c_\gamma]$ for $\rho$. If we set $c_\gamma \leq 1/\sqrt{2}$, then we have
\begin{lemma}\label{lem:PhiLips}
    There is a constant $c>0$ such that the Lipschitz constant of $\phi$ when $\alpha\in [-1/\sqrt{2}, 0]$ and $\rho$ when $\alpha\in \left(\gamma_i,1/\sqrt{2}\right]$ is less than
    $$ c\exp\left(-\gamma_{i+1}^2 k/c\right)\sqrt{k},$$
    for $k\ge c \gamma_{i+1}^{-2}$.
\end{lemma}
Combining this result with Lemma~\ref{lem:easyLip} completes the proof of Lemma~\ref{lem:finalLip}.

To prove Lemma~\ref{lem:PhiLips}, we need to bound the Lipschitz constants of $\phi$ when $\alpha\in [-1/\sqrt{2}, 0]$ and $\rho$ when $\alpha\in \left(\gamma_i,1\sqrt{2}\right]$.
We will go through the details for $\phi$, and comment how the argument for $\rho$ differs along the way.

First recall the following claim
\begin{customclm}{\ref{clm:distDeterm}}
For any $(x,y) \in \finalX \times \{-1,1\}$ and any $w \in \finalH$, the distribution of $y \langle h_{\rA,\rt}(w),\rA x \rangle$ is completely determined from $y \langle w, x \rangle$. 
\end{customclm}

As we need to understand the distribution of the random variable $y \langle h_{\rA,\rt}(w),\rA x \rangle$ to bound the Lipschitz constants of $\phi$ and $\rho$, we proceed to give the proof of Claim~\ref{clm:distDeterm} while introducing convenient notation for establishing Lemma~\ref{lem:PhiLips}.

\begin{proof}
Firstly, write $h_{\rA,\rt}(w) =\rA w + \rv$ with $\rv =(h_{\rA,\rt}(w)-\rA w)$. Then observe that $(\rA w)_i = \langle \ra_i, w\rangle \sim \Norm(0, \|w\|^2/k) \eqd \Norm(0,1/k)$ where $\ra_i$ denotes the $i$'th row of $\rA$. Here $\eqd$ denotes equality in distribution. Now write $x = \langle w, x \rangle w + u$ where $\langle u, w \rangle = 0$ and $\|u\|^2 = \|x\|^2-\langle w,x \rangle^2 = 1-\langle w,x \rangle^2$ (i.e.\ a Gram-Schmidt step). We have $(\rA x)_i = \langle \ra_i , w\rangle \langle w,x\rangle + \langle \ra_i , u\rangle$. By rotational invariance of the Gaussian distribution and orthogonality of $w$ and $u$, we have that $\langle \ra_i , u \rangle \sim \Norm(0,(1-\langle w, x\rangle^2)/k)$ and that this is independent of $\langle \ra_i, w \rangle$. Using the independence, we also conclude that if we condition on any fixed outcome of $c_i = (\rA w)_i$, we have that $c_i \langle \ra_i, u \rangle$ is $\Norm(0, c^2_i (1-\langle w , x \rangle^2)/k)$ distributed.

We now argue that we can sample from the distribution of $y \langle h_{\rA,\rt}(w), \rA x \rangle = \langle h_{\rA,\rt}(w), y \rA x \rangle$ knowing only $y \langle w, x \rangle$ as follows: Sample independent $\Norm(0,1/k)$ distributed random variables $\rX_1,\dots,\rX_k$. Next sample independent $\Norm(0,(1-y^2 \langle w, x\rangle^2)/k)$ distributed random variables $\rY_1,\dots,\rY_k$ and let $\rZ_i = y \langle w, x \rangle \rX_i + \rY_i \eqd y \langle w, x \rangle \rX_i + y\rY_i$, where the last step follows from independence of $\rX_i$ and $\rY_i$ and symmetry in the distribution of $\rY_i$. Let $\rX$ be the vector with the $\rX_i$'s as entries and $\rZ$ similarly. Then the joint distribution of $(\rX, \rZ)$ is equal to the joint distribution of $(\rA w, y\rA x)$. Finally draw offsets $\rt'_1,\dots,\rt'_k$ uniformly and independently in $[0,1]$ and round $\rX_i$ to a number of the form $(1/2)(10\sqrt{k})^{-1} + z(10\sqrt{k})^{-1}$ for $z \in \Ints$ as in the definition of $h_{\rA,\rt}$. The resulting variables $\rX_i'$ satisfy that $y \langle h_{\rA, \rt}(w), \rA x \rangle \eqd \langle \rX', \rZ \rangle$.
\end{proof}

With Claim~\ref{clm:distDeterm} established, we will use the notation in the proof as we proceed with bounding the Lipschitz constants of $\phi$ and $\rho$.

Let $\alpha = y\ipr{w,x}$ for some $w \in \finalH$ and $(x,y) \in \finalX \times \{-1,1\}$ and $\alpha \in [-1/\sqrt{2}, 0]$ (for $\rho$, let $\alpha \in (\gamma_i, 1/\sqrt{2}]$). Let $\rX_i\sim \Norm(0,1/k)$, $\rY_i\sim\Norm(0,(1-\alpha^2)/k)$ and let $\rX_i'$ be the random rounding of $\rX_i$. We argued, in the proof of Claim~\ref{clm:distDeterm}, that
$y\ipr{h_{\rA,\rt}(w),\rA x} \eqd \ipr{\rX',\alpha \rX + \rY}$.
Let additionally $E_i$ be the event that $\rX_i'$ is rounded up. For notational convenience, let $\rM_i = \sqrt{k} \rX_i$ and observe that $\rM_i\sim \Norm(0,1)$. With this notation, we have that $\rX_i'$ has the form
\begin{align*}
    \rX_i'
    = \frac{1}{10\sqrt{k}}\left(\left\lfloor{10 \rM_i-\frac{1}{2\sqrt{k}10}}\right\rfloor+\indi{E_i}+\frac{1}{2}\right).
\end{align*}
Hence, 
$$y\ipr{h_{\rA,\rt}(w),\rA x} 
\eqd \ipr{\rX',\alpha \rX + \rY}
= \frac{\alpha}{\sqrt{k}}\ipr{\rX',\rM}+\ipr{\rX',\rY}.$$
Recall that the variables $\rt_i$ and $\rM_i$ determine $\rX_i, E_i$ and thus also $\rX'_i$. If we condition on an outcome $\rt_i=t_i$ and $\rM_i=M_i$, only $\rY_i$ remains random. We may thus write
\begin{align*}
  \Pr[y\langle h_{\rA,\rt}(w),\rA x\rangle > \gamma_i/2 ] &=\\ \Pr\left[\frac{\alpha}{\sqrt{k}}\ipr{\rX',\rM}+\ipr{\rX',\rY} > \gamma_i/2\right]
  &=\\
  \int_{\R^k\times [0,1]^k} f_{\rM,\rt}(M,t)\Pr\left[\frac{\alpha}{\sqrt{k}}\ipr{\rX',\rM}+\ipr{\rX',\rY} >\gamma_i/2 \ \bigg| \  \rM_i = M_i, \rt_i = t_i\right] d(M,t)
  &=\\ \int_{\R^k\times[0,1]^k} \Pr\left[\frac{\alpha}{\sqrt{k}}\ipr{X',M}+\ipr{X',\rY} > \gamma_i/2\right] d(M,t),
  \end{align*}
where $f_{M,t}(M,t)$ is the joint probability density function of $\rM$ and $\rt$. 

Let us now define $\rN_i$ such that $\rY_i = \sqrt{\frac{1-\alpha^2}{k}}\rN_i$ and let $\rN = (\rN_1,\dots,\rN_k)$. Then $\rN_i\sim \Norm(0,1)$ and the event 
\[
\frac{\alpha}{\sqrt{k}}\ipr{X',M}+\ipr{X',\rY} > \gamma_i/2,
\]
may be rewritten as
\begin{align*}
    \frac{\alpha}{\sqrt{k}}\ipr{X',M}+\ipr{X',\rY} > \gamma_i/2 \Longleftrightarrow
    \ipr{X',\rY} &> \gamma_i/2 - \frac{\alpha}{\sqrt{k}}\ipr{X',M} \Longleftrightarrow\\
    \sqrt{\frac{1-\alpha^2}{k}}\ipr{X', \rN} &> \gamma_i/2 - \frac{\alpha}{\sqrt{k}}\ipr{X',M} \Longleftrightarrow\\
    \sqrt{\frac{1-\alpha^2}{k}}\norm{X'}_2 \ipr{X'/\|X'\|_2,\rN} &> \gamma_i/2 - \frac{\alpha}{\sqrt{k}}\ipr{X',M} \Longleftrightarrow\\
    \ipr{X'/\|X'\|_2,\rN} &> \frac{\sqrt{k}\gamma_i/2 - \alpha\ipr{X',M}}{\sqrt{1-\alpha^2}\norm{X'}_2}.
\end{align*}
Observe that $\ipr{X'/\|X'\|_2, \rN} \sim \Norm(0,1)$. If we let $\Phi$ denote the cumulative density function of a standard normal distribution, then we have established
\begin{align}
  \Pr[y\langle h_{\rA,\rt}(w),\rA x\rangle > \gamma_i/2] &= \int_{\R^k\times[0,1]^k} f_{M,t}(M,t)\left(1-\Phi\left(\frac{\sqrt{k}\gamma_i/2 - \alpha\ipr{X',M}}{\sqrt{1-\alpha^2}\norm{X'}_2}\right)\right) d(M,t) \nonumber \\
  &=\int_{\R^k\times[0,1]^k} f_{\rM}(M)\left(1-\Phi\left(\frac{\sqrt{k}\gamma_i/2 - \alpha\ipr{X',M}}{\sqrt{1-\alpha^2}\norm{X'}_2}\right)\right) d(M,t).
  \label{eq:integral}
  \end{align}
In the last equality, we use that $\rM$ and $\rt$ are independent and that the probability density function of $\rt$ is $1$ since each $\rt_i$ is uniform in $[0,1]$. This reduces $f_{M,t}(M,t)$ to the probability density function $f_{\rM}(M)$ of $\rM$ alone.
  
The same arguments for $\rho$ also gives the integral~\eqref{eq:integral}, with the small difference that $(1-\Phi(\cdot))$ is replaced by $\Phi(\cdot)$. This difference is irrelevant, since to bound the Lipschitz constant, we will differentiate and bound the differential's absolute value.

Let $g(M,t,\alpha)$ be the integrant above, we want to differentiate $\int_{\R^k\times[0,1]^k} g(M,t,\alpha) \enspace d(M,t)$ by differentiating under the integral. Standard measure theory results (Theorem 6.28, \cite{Klenke:2013}) allows us to do this if we satisfy three conditions. These conditions are, in this case, equivalent to the following
\begin{enumerate}
    \item[i)] for all constant $\alpha$
    , the integral $\int_{\R^k\times[0,1]^k} g(M,t,\alpha) \enspace d(M,t)$ is finite.
    \item[ii)] for all constant $M,t$, the partial differential of $g(M,t,\alpha)$ with respect to $\alpha$ exists.
    \item[iii)] There exists a function $h(M,t)$, where $\int_{\R^k\times[0,1]^k} h(M,t) \enspace d(M,t)$ is finite and such that $|\frac{\partial}{\partial\alpha}g(M,t,\alpha)|\le h(M,t)$ for all $\alpha$.
\end{enumerate}
The first two conditions are straightforward: The integral is equal to a probability, which is finite. And $g$ is a combination of differentiable functions making it differentiable itself.
The last condition is more cumbersome, but the goal of this proof is to upperbound the integral by a constant, which clearly dosn't depend on $\alpha$. Hence the last condition will be satisfied during the proof. 

Hence we can continue with our differentiation by differentiating under the integral.
\begin{align*}
    \bigg| \frac{\partial}{\partial\alpha}\Pr[y \ipr{h_{\rA,\rt}(w),\rA x}> \gamma_i/2 ] \bigg| 
    &= \\\left| \int_{\R^k\times[0,1]^k} \dfrac{\partial}{\partial\alpha}f_{\rM}(M)\left(1-\Phi\left(\frac{\sqrt{k}\gamma_i/2 - \alpha\ipr{X',M}}{\sqrt{1-\alpha^2}\norm{X'}_2}\right)\right) d(M,t) \right| 
    &= \\ \frac{1}{2\pi}\left| \int_{\R^k\times[0,1]^k} f_{\rM}(M)\exp\left(-\frac{1}{2}\left(\frac{\sqrt{k}\gamma_i/2 - \alpha\ipr{X',M}}{\sqrt{1-\alpha^2}\norm{X'}_2}\right)^2 \right)\left(\frac{\frac{\alpha\sqrt{k}\gamma_i}{2} - \ipr{X',M}}{(1-\alpha^2)^{3/2}\norm{X'}_2}\right) \enspace d(M,t) \right|
    &\le \\ \frac{1}{2\pi}\int_{\R^k\times[0,1]^k} f_{\rM}(M)\exp\left(-\frac{(\sqrt{k}\gamma_i/2 - \alpha\ipr{X',M})^2}{2(1-\alpha^2)\norm{X'}_2^2}\right)\left|\frac{\frac{\alpha\sqrt{k}\gamma_i}{2} - \ipr{X',M}}{(1-\alpha^2)^{3/2}\norm{X'}_2} \right| \enspace d(M,t).
\end{align*}
For both $\alpha \in [-1/\sqrt{2}, 0]$ and $\alpha \in \left(\gamma_i, 1/\sqrt{2}\right]$, we have that $1\ge (1-\alpha^2) \geq 1/2$ and thus, for both $\phi$ and $\rho$, the above is upper bounded by
\begin{align*}
\frac{2^{3/2}}{2 \pi} \cdot &\int_{\R^k\times[0,1]^k} f_{\rM}(M)\exp\left(-\frac{(\sqrt{k}\gamma_i/2 - \alpha\ipr{X',M})^2}{2\norm{X'}_2^2}\right)\frac{\sqrt{k}\gamma_i + \left|\ipr{X',M}\right| }{\norm{X'}_2} \enspace d(M,t)  \\
&\le
\int_{\R^k\times[0,1]^k} f_{\rM}(M)\exp\left(-\frac{(\sqrt{k}\gamma_i/2 - \alpha\ipr{X',M})^2}{2\norm{X'}_2^2}\right)\left(\frac{\sqrt{k}\gamma_i}{\norm{X'}_2} + \|M\|_2 \right)\enspace d(M,t).
\end{align*}
We now use that $|X'_i| \geq (1/2)(10 \sqrt{k})^{-1}$ for all $i$. This implies $\|X'\|_2 \geq \sqrt{k (1/4)(10 \sqrt{k})^{-2}} = 1/20$. We may thus further upper bound the above by
\begin{align}
20 \cdot \int_{\R^k\times[0,1]^k} f_{\rM}(M)\exp\left(-\frac{(\sqrt{k}\gamma_i/2 - \alpha\ipr{X',M})^2}{2\norm{X'}_2^2}\right)\left(\sqrt{k}\gamma_i + \|M\|_2 \right)\enspace d(M,t).\label{eq:finalintegral}
\end{align}
We will bound~\eqref{eq:finalintegral}, by splitting it into 3 cases: 
\begin{align*}
\textit{i)}\enspace\norm{M}_2^2\le \frac{9}{10}k\qquad\qquad\qquad
\textit{ii)}\enspace\frac{9}{10}k\le\norm{M}_2^2\le \frac{4}{3}k\qquad\qquad\qquad
\textit{iii)}\enspace\norm{M}_2^2\ge \frac{4}{3}k.
\end{align*}
The arguments for cases \textit{i)} and \textit{iii)} do not depend on $\alpha$, and hence are identical for $\rho$ and $\phi$. In those cases, we simply exploit that $\|\rM\|_2^2 \sim \chi^2_k$ and thus these cases are very unlikely. This implies that the integral over $f_M(M)$ is so small that we can afford to upper bound the exponential term in~\eqref{eq:finalintegral} by 1. For case \textit{ii)}, we can use the assumptions on $\|M\|_2^2$ to show that the exponential term is no more than $c\exp(-\gamma_i^2 k/c)$ for a constant $c>0$. We proceed to the three cases.

\paragraph{case \textit{i)}.}
We simply upper bound the exponential term in~\eqref{eq:finalintegral} by $1$ and use the assumption that $\norm{M}_2^2\le \frac{9}{10}k$ to conclude
\begin{align*}
\exp\left(-\frac{(\sqrt{k}\gamma_i/2 - \alpha\ipr{X',M})^2}{2\norm{X'}_2^2}\right)\left(\sqrt{k}\gamma_i + \|M\|_2 \right)
\le 2\sqrt{k}.
\end{align*}
Now since $\rM$ is multivariate standard normal, $\norm{\rM}_2^2$ is $\chi_k^2$ distributed. Let $\rZ \sim \chi_k^2$ with probability density function $f_Z(z)$. Then the integral in~\eqref{eq:finalintegral} in is bounded by:
\begin{align*}
40\sqrt{k} \int_{(\sqrt{9 k/10})\Ball_2^k} f_{\rM}(M)\enspace dM
= 40\sqrt{k} \int_0^{9 k/10} f_{Z}(z)\enspace dz 
= 40\sqrt{k} \cdot \Pr[\rZ<9k/10],
\end{align*}
which by Theorem~\ref{thm:chibound} is less than
$$80\sqrt{k} \exp\left(-k\gamma_i^2/800\right).$$
\paragraph{case \textit{ii)}.}
We use the assumption that $\frac{9}{10}k\le\norm{M}_2^2\le \frac{4}{3}k$ together with the following observations
\begin{remark}\label{rem:case_ii_norm_bound}
   If $\|X\|_2^2 \leq 4/3$, then $\norm{X'}_2^2 < 2$.
\end{remark}
\begin{remark}\label{rmk:ipXX'}
    If $\norm{X}_2^2\ge 9/10$, then $(8/9)\norm{X}_2^2\le \ipr{X,X'}\le (10/9)\norm{X}_2^2$.
\end{remark}
We prove Remark~\ref{rem:case_ii_norm_bound} and Remark~\ref{rmk:ipXX'} in Appendix~\ref{sec:aux}.
Since $\sqrt{k}X = M$, Remark~\ref{rmk:ipXX'} gives $\ipr{X',M}\ge (8/10) \sqrt{k}$ 
and hence
\begin{alignat*}{3}
    \alpha > \gamma_i 
    &\Longrightarrow\quad  \frac{\sqrt{k}\gamma_i}{2} - \alpha\ipr{X',M}
    \le-\frac{3\sqrt{k}\gamma_i}{10}
    \le 0
    &&\Longrightarrow\quad 
    -\left(\frac{\sqrt{k}\gamma_i}{2} - \alpha\ipr{X',M}\right)^2
    \le -\frac{9k\gamma_i^2}{100}\\
    \alpha < 0 
    &\Longrightarrow\quad
    \frac{\sqrt{k}\gamma_i}{2} - \alpha\ipr{X',M} 
    \ge\frac{\sqrt{k}\gamma_i}{2}
    \ge 0 
    &&\Longrightarrow\quad
    -\left(\frac{\sqrt{k}\gamma_i}{2} - \alpha\ipr{X',M}\right)^2
    \le -\frac{k\gamma_i^2}{4}.
\end{alignat*}
Hence for both $\phi$ and $\rho$, the last two factors of the integral in~\eqref{eq:finalintegral} are bounded by:
\begin{align*}
\exp\left(-\frac{(\sqrt{k}\gamma_i/2 - \alpha\ipr{X',M})^2}{2\norm{X'}_2^2}\right)\left(\sqrt{k}\gamma_i + \norm{M}_2 \right)
\le\frac{5}{2}\sqrt{k}\exp\left(-\frac{9 k\gamma_i^2}{100 \cdot 4}\right).
\end{align*}
Which gives the following:
\begin{align*}
50\sqrt{k}\exp\left(-\frac{k\gamma_i^2}{50}\right) \int_{(\sqrt{9 k/10} \cdot \Ball_2^k)^C\cap (\sqrt{4 k/3} \cdot \Ball_2^k)} f_{\rM}(M)\enspace dM
\le 50\sqrt{k}\exp\left(-\frac{k\gamma_i^2}{50}\right).
\end{align*}

\paragraph{case \textit{iii)}.}
We bound the last two factors of the integral~\eqref{eq:finalintegral} under the assumption that $\norm{M}_2^2\ge \frac{4}{3}k$. Here we simply upper bound the exponential by $1$ and get
\begin{align*}
\exp\left(-\frac{(\sqrt{k}\gamma_i/2 - \alpha\ipr{X',M})^2}{2\norm{X'}_2^2}\right)\left(\sqrt{k}\gamma_i + \|M\|_2 \right)
\le 2\norm{M}_2,
\end{align*}
hence the integral~\eqref{eq:finalintegral} is bounded by:
\begin{align}
40 \int_{(\sqrt{4 k/3} \cdot \Ball_2^k)^C} f_{\rM}(M)\norm{M}_2 \enspace dM \label{eq:intz}.
\end{align}
Recall that $\norm{\rM}_2^2\sim\chi_k^2$ and let $\rZ\sim \chi_k^2$ with probability density function $f_Z(z)$. Then the integral~\eqref{eq:intz} is equal to
\begin{align*}
    40\int_{4k/3}^\infty f_{Z}(z)\sqrt{z} \enspace dz.
\end{align*}
Let also $L_i = \left[\frac{4}{3}k\cdot2^{i}, \frac{4}{3}k\cdot2^{i+1} \right)$  for $i\in\Z_{\ge0}$. By definition, the $L_i$'s partition $\left[\frac{4}{3}k,\infty\right)$, and we upper bound with: 
\begin{align*}
    40\sum_{i=0}^\infty \int_{L_i} f_{Z}(z)\sqrt{\frac{4}{3}k\cdot2^{i+1}} \enspace dz
    &\le 70 \sqrt{k}\sum_{i=0}^\infty \Pr[\rZ\in L_i] 2^{i/2}
    \le 70 \sqrt{k} \sum_{i=0}^\infty \Pr\left[\rZ\ge \frac{4}{3}k\cdot2^{i}\right] 2^{i/2}.
\end{align*}
Using the following remark:
\begin{remark}[\cite{LaurentMassart}, equation 4.3, page 1326]\ \newline
    Let $\rZ\sim \chi^2_k$, $y > 0$ then:
    $$\Pr[\rZ\ge2\sqrt{ky}+2y+k]\le \exp(-y).$$
\end{remark}
With $y= c\frac{4}{3} k$, where $c = \frac{1}{8^2}2^i$. We have: 
\begin{align*}
    2\sqrt{ky}+2y+k 
    &= \left(\sqrt{c}\frac{4}{\sqrt{3}}+c\frac{8}{3}+1\right)k
    \le \left(\frac{4\sqrt{3}}{24}+\frac{1}{24}+1\right)k\cdot 2^i
    \le \frac{4}{3}k\cdot 2^i.
\end{align*}
Hence we can finish the bound for this case, using the assumption that $k\ge \gamma_i^{-2} 72 \ln(2)$
\begin{align*}
    70\sqrt{k} \sum_{i=0}^\infty \Pr\left[\rZ\ge \frac{4}{3}k\cdot2^{i}\right] 2^{i/2}
    &\le70\sqrt{k} \sum_{i=0}^\infty \exp\left(-\frac{1}{48} k \cdot 2^i\right) 2^{i/2}\\
    &\le70\sqrt{k}\exp\left(-\frac{1}{48} k\right) \sum_{i=0}^\infty \exp\left(-\frac{1}{48} k \right)^{i} 2^{i/2}\\
    &\le140\sqrt{k}\exp\left(-\frac{1}{48} k\gamma_i^2\right).
\end{align*}

\paragraph{Collecting the three cases.}
Hence in total $|\frac{\partial}{\partial\alpha}\phi|$ for $\alpha\in [-1/\sqrt{2}, 0]$ and $|\frac{\partial}{\partial\alpha}\rho|$ for $\alpha\in \left(\gamma_i,1\sqrt{2}\right]$ are bounded by:
\begin{align*}
    80\sqrt{k} \exp\left(- k\gamma_i^2/800\right)
    + 50\sqrt{k}\exp\left(-\frac{k\gamma_i^2}{50}\right)
    + 140\sqrt{k}\exp\left(-\frac{1}{48} k\gamma_i^2\right)\\
    \le 140\exp\left(-k\gamma_i^2/800\right).
\end{align*}
Using that $\gamma_{i+1}=2\gamma_i$, this completes the proof of Lemma~\ref{lem:PhiLips}.

%% file: grid.tex
\section{Meet in the Middle Bound}
\label{sec:meetinmid}
The goal of this section is to prove the following
\begin{customlem}{\ref{lem:supW}}
There is a constant $c>0$ such that with probability at least $1-\delta$ over $\rS \sim \cD^n$ we have
\begin{align*}
    \sup_{w \in \subH(\Gamma_i,L_j)} \left|\E_{\rA,\rt} [\Loss^{\gamma_i/2}_{\rA \cD}(h_{\rA,\rt}(w)) - \Loss^{\gamma_i/2}_{\rA \rS}(h_{\rA,\rt}(w))]\right| &\leq\\ c \left(\sqrt{\frac{(\ell_{j+1}+ \exp(-\gamma_{i+1}^2k/c)) (k + \ln(e/\delta))}{n}} + \frac{(k + \ln(e/\delta))}{n} \right).
\end{align*}
\end{customlem}
Notice here that the two losses $\Loss^{\gamma_i/2}_{\rA \cD}(h_{\rA,\rt}(w))$ and $\Loss^{\gamma_i/2}_{\rA \rS}(h_{\rA,\rt}(w))$ refer to the same margin $\gamma_i/2$ and $h_{\rA,\rt}(w)$ has been discretized to have all coordinates of the form $(1/2)(10 \sqrt{k})^{-1} + z (10 \sqrt{k})^{-1}$ for integer $z$. Intuitively, we will try to exploit this discretization to union bound over a grid of finitely many hypotheses. Unfortunately, the random matrix $\rA$ may increase the norm of $w$ arbitrarily much, and thus a single grid is insufficient. Instead, we need an infinite sequence of grids. For this, let $\Disc_0$ denote the set of all vectors in $4 \Ball_2^k$ whose coordinates are of the form $(1/2)(10 \sqrt{k})^{-1} + z(10 \sqrt{k})^{-1}$ for integer $z$. More generally, let $\Disc_i$ for $i > 0$ denote the set of all vectors in $(2^i \cdot 4 \Ball_2^k)$ whose coordinates are of this form. Since $\|x\|_1 \leq \sqrt{k} \|x\|_2$ for any $x \in \R^k$, we have that $\Disc_i \subset (2^i \cdot 4 \Ball_2^k) \subseteq \sqrt{k}(2^i \cdot 4 \Ball_1^k)$. For a vector $x \in \Disc_i$, let $i(x)=(i_1,\dots,i_k)$ denote the integers so that $x = (10 \sqrt{k})^{-1} i(x) + (1/2)(10 \sqrt{k})^{-1} \AllOne$ with $\AllOne \in \R^k$ the all-1's vector. Then by the triangle inequality, we have $(10 \sqrt{k})^{-1}\|i(x)\|_2 \leq \|x\|_2 + (1/2)(10 \sqrt{k})^{-1}\|\AllOne\|_2 \leq 2^i \cdot 4 + 1/20$. This implies $\|i(x)\|_1 \leq (10 \sqrt{k}) \sqrt{k} (2^i \cdot 4 + 1/20) \leq (5 \cdot 2^{i+3}+1)k$. Since each coordinate of $i(x)$ is an integer, there are thus at most $2^k$ choices for the signs and $\sum_{t=0}^{(5 \cdot 2^{i+3}+1)k} \binom{k + t -1}{t}$ choices for the absolute values of the integers. That is, we have
\begin{align}
|\Disc_i| \leq 2^k \cdot \sum_{t=0}^{(5 \cdot 2^{i+3}+1) k} \binom{k + t -1}{t} \leq 2^{(5 \cdot 2^{i+3}+3)k} \leq 2^{2^{i+7}k}.\label{eq:netsize}
\end{align}
We now start by considering a fixed outcome $A$ of the random matrix $\rA$. For such a fixed $A$, the training set $\rS$ behaves well in the sense that $\Loss^\gamma_{A \cD}(w)$ and $\Loss^\gamma_{A \rS}(w)$ are close with high probability for any $w$. This is formalized in the following remark
\begin{remark}
\label{rmk:concentrationx}
For any distribution $\cD$ over $\finalX \times \{-1,1\}$, fixed $w \in \finalH$, margin $\gamma$ and any $A \in \R^{k \times d}$, it holds with probability at least $1-\delta$ over $\rS \sim \cD^n$ that
\[
|\Loss_{A\cD}^{\gamma}(w) - \Loss_{A\rS}^{\gamma}(w)| \leq \sqrt{\frac{8\Loss_{A \cD}^{\gamma}(w)\ln(1/\delta)}{n}} + \frac{2 \ln(1/\delta)}{n}.
\]
\end{remark}
The proof of Remark~\ref{rmk:concentrationx} is a simple application of Bernstein's and can be found in Appendix~\ref{sec:aux}.

In Lemma~\ref{lem:supW}, the matrix $\rA$ is not fixed but random. Thus we need to find a formal property of the training set $\rS$ under which $\Loss^{\gamma_i/2}_{\rA \cD}(h_{\rA,\rt}(w))$ and $\Loss^{\gamma_i/2}_{\rA \rS}(h_{\rA,\rt}(w))$ are close in expectation over the random choice of $\rA$. With this goal in mind, we now say that a matrix $A$ in the support of $\rA$ and a training set $S$ has \emph{distortion} at least $\beta$, if there is a grid $\Disc_a$ and a vector $w \in \Disc_a$ such that
\[
|\Loss_{A\cD}^{\gamma_i/2}(w) - \Loss_{A\rS}^{\gamma_i/2}(w)| > \beta \cdot \left(\sqrt{\frac{8\Loss_{A \cD}^{\gamma_i/2}(w)(2^{a+7}k + \ln(1/\delta))}{n}} + \frac{2 (2^{a+7}k + \ln(1/\delta))}{n}\right).
\]
For a training set $S$, we use $D_\beta(S)$ to denote the set of matrices $A$ with distortion at least $\beta$ for $S$.

We observe that for a fixed matrix $A$, grid $\Disc_a$ and $\beta>1$, we have by Remark~\ref{rmk:concentrationx} with $\delta'_a = (\delta/2^{2^{a+7}k})^{\beta}$
and a union bound over all $w \in \Disc_a$, that with probability at least $1-|\Disc_a|\delta'_a$, it holds for all $w \in \Disc_a$ that
\begin{align*}
    |\Loss_{A\cD}^{\gamma_i/2}(w) - \Loss_{A\rS}^{\gamma_i/2}(w)| &\leq \sqrt{\frac{8\Loss_{A \cD}^{\gamma_i/2}(w)\ln(1/\delta'_a)}{n}} + \frac{2 \ln(1/\delta'_a)}{n} \\
    &= \sqrt{\frac{8\Loss_{A \cD}^{\gamma_i/2}(w)(\beta 2^{a+7}k + \beta \ln(1/\delta))}{n}} + \frac{2 (\beta 2^{a+7}k + \beta \ln(1/\delta))}{n}  \\
    &\leq \beta \cdot \left(\sqrt{\frac{8\Loss_{A \cD}^{\gamma_i/2}(w)(2^{a+7}k + \ln(1/\delta))}{n}} + \frac{2 (2^{a+7}k + \ln(1/\delta))}{n}\right).
\end{align*}
Thus for $\beta \geq 2$, we have
\begin{align*}
\Pr_{\rS}[A \in D_\beta(\rS)] &\leq \sum_{a=0}^\infty |\Disc_a| \delta'_a \\
&\leq \sum_{a=0}^\infty \delta^\beta \cdot 2^{-(\beta-1)2^{a+7}k} \\
&\leq 2 \cdot \delta^{\beta} \cdot 2^{-(\beta-1)2^{7}k}.
\end{align*}
By Markov's inequality, we have
\begin{align*}
\Pr_{\rS}[\Pr_{\rA}[\rA \in D_\beta(\rS)] > 2 \cdot \delta^{\beta/2} \cdot 2^{-(\beta-1)\cdot 2^{6}k}] &\leq \frac{\E_{\rS}[\Pr_{\rA}[\rA \in D_\beta(\rS)]}{2 \cdot \delta^{\beta/2} \cdot 2^{-(\beta-1)\cdot 2^{6}k}} \\
&= \frac{\E_{\rA}[\Pr_{\rS}[\rA \in D_\beta(\rS)]}{2 \cdot \delta^{\beta/2} \cdot 2^{-(\beta-1)\cdot 2^{6}k}}\\
&\leq \delta^{\beta/2} \cdot 2^{-(\beta-1)2^{6}k}.
\end{align*}
Now call a training set $S$ \emph{representative} if it holds for every $\beta=2^h$ with integer $h \geq 1$ that
\[
\Pr_{\rA}[\rA \in D_\beta(\rS)] \leq 2 \cdot \delta^{\beta/2} \cdot 2^{-(\beta-1)\cdot 2^{6}k}.
\]
A union bound implies that $\rS$ is representative with probability at least
\[
1-\sum_{h=1}^\infty 2 \cdot \delta^{2^{h-1}} \cdot 2^{-(2^h-1)2^{6}k} \geq 1-\frac{\delta}{2^{2^6 k-2}} \geq 1-\delta.
\]
Now define for integer $h \geq 1$ the set
\[
K_h(S) = D_{2^h}(S) \setminus \left(\cup_{b=h+1}^{\infty} D_{2^b}(S) \right).
\]
Let $K_0(S)$ be defined as
\[
K_0(S) = \support(\rA) \setminus \left(\cup_{b=1}^{\infty} D_{2^b}(S) \right).
\]

For any $w \in \finalH$, we may use the triangle inequality to conclude
\begin{align*}
 \left|\E_{\rA,\rt} [\Loss^{\gamma_i/2}_{\rA \cD}(h_{\rA,\rt}(w)) - \Loss^{\gamma_i/2}_{\rA S}(h_{\rA,\rt}(w))]\right| &\leq \\
 \sum_{h=0}^\infty \E_{\rA,\rt}\left[\left|\Loss^{\gamma_i/2}_{\rA \cD}(h_{\rA,\rt}(w)) - \Loss^{\gamma_i/2}_{\rA S}(h_{\rA,\rt}(w))\right| \mid \rA \in K_h(S)\right] \Pr_{\rA}[\rA \in K_h(S)].
\end{align*}
Now consider an $A \in K_h(S)$. Then $A$ has distortion no more than $2^{h+1}$ by definition of $K_h(S)$. This implies that if $h_{A,t}(w)$ is in $\Disc_a$ but not $\Disc_b$ for $b < a$, then $\|h_{A,t}(w)\|_2 \geq 2^{a+1}$ by definition of $\Disc_b$ and we get
\begin{align*}
|\Loss_{A\cD}^{\gamma_i/2}(h_{A,t}(w)) - \Loss_{A\rS}^{\gamma_i/2}(h_{A,t}(w))| &\leq\\
2^{h+1} \cdot \left(\sqrt{\frac{8\Loss_{A \cD}^{\gamma_i/2}(w)(2^{a+7}k + \ln(1/\delta))}{n}} + \frac{2 (2^{a+7}k + \ln(1/\delta))}{n}\right) 
&\leq \\
2^{h+8}  \|h_{A,t}(w)\|_2 \cdot \left(\sqrt{\frac{8\Loss_{A \cD}^{\gamma_i/2}(w)(k + \ln(1/\delta))}{n}} + \frac{2 (k + \ln(1/\delta))}{n}\right).
\end{align*}
Using Cauchy-Schwartz, we thus get for any $w \in \finalH$ that
\begin{align*}
 \left|\E_{\rA,\rt} [\Loss^{\gamma_i/2}_{\rA \cD}(h_{\rA,\rt}(w)) - \Loss^{\gamma_i/2}_{\rA S}(h_{\rA,\rt}(w))]\right| &\leq \\
\sum_{h=0}^\infty 2^{h+8} \E_{\rA,\rt}\bigg[\|h_{\rA,\rt}(w)\|_2  \cdot  \bigg(\sqrt{\frac{8\Loss_{\rA \cD}^{\gamma_i/2}(w)(k + \ln(1/\delta))}{n}} &+\\
\frac{2 (k + \ln(1/\delta))}{n}\bigg)\mid \rA \in K_h(S)\bigg]\Pr_{\rA}[\rA \in K_h(S)]&\leq \\
\sum_{h=0}^\infty 2^{h+8} \sqrt{\E_{\rA,\rt}\left[\|h_{\rA,\rt}(w)\|^2_2  \mid \rA \in K_h(S) \right]} &\ \cdot \\
\sqrt{\E_{\rA,\rt}\left[ \left(\sqrt{\frac{8\Loss_{\rA \cD}^{\gamma_i/2}(w)(k + \ln(1/\delta))}{n}} + \frac{2 (k + \ln(1/\delta))}{n}\right)^2\mid \rA \in K_h(S)\right]}\Pr_{\rA}[\rA \in K_h(S)].
\end{align*}
By Cauchy-Schwartz, this is at most
\begin{align*}
\sqrt{\sum_{h=0}^\infty 2^{2h+16}\E_{\rA,\rt}[\|h_{\rA,\rt}(w)\|_2^2 \mid \rA \in K_h(S) ]  \Pr_{\rA}[\rA \in K_h(S)]} &\ \cdot\\
\sqrt{\sum_{h=0}^\infty \E_{\rA,\rt}\left[ \left(\sqrt{\frac{8\Loss_{\rA \cD}^{\gamma_i/2}(w)(k + \ln(1/\delta))}{n}} + \frac{2 (k + \ln(1/\delta))}{n}\right)^2\mid \rA \in K_h(S)\right]\Pr_{\rA}[\rA \in K_h(S)] }.
\end{align*}
Using Cauchy-Schwartz again and Jensen's inequality, the first sum is bounded by
\begin{align*}
\sum_{h=0}^\infty 2^{2h+16} \E_{\rA,\rt}[\|h_{\rA,\rt}(w)\|^2_2 \mid \rA \in K_h(S) ] \Pr_{\rA}[\rA \in K_h(S)] &\leq \\
\sqrt{ \sum_{h=0}^\infty 2^{4h + 64}\Pr_{\rA}[\rA \in K_h(S)] } \cdot \sqrt{\sum_{h=0}^\infty \E_{\rA,\rt}[\|h_{\rA,\rt}(w)\|^2_2 \mid \rA \in K_h(S) ]^2 \Pr_{\rA}[\rA \in K_h(S)] } &\leq \\
\sqrt{ \sum_{h=0}^\infty 2^{4h + 64}\Pr_{\rA}[\rA \in D_{2^h}(S)] } \cdot \sqrt{\sum_{h=0}^\infty \E_{\rA,\rt}[\|h_{\rA,\rt}(w)\|^4_2 \mid \rA \in K_h(S) ] \Pr_{\rA}[\rA \in K_h(S)] } &\leq \\
\sqrt{ \sum_{h=0}^\infty 2^{4h + 64} 2 (\delta/2^{2^7k+1})^{(2^h-1)/2} } \cdot \sqrt{\E_{\rA,\rt}[\|h_{\rA,\rt}(w)\|^4_2 ] } &\leq \\
2^{33} \cdot \sqrt{\E_{\rA,\rt}[\|h_{\rA,\rt}(w)\|^4_2 ] }.
\end{align*}
Using Jensen's inequality on the second sum, we find that
\begin{align*}
\sum_{h=0}^\infty \E_{\rA,\rt}\left[ \left(\sqrt{\frac{8\Loss_{\rA \cD}^{\gamma_i/2}(w)(k + \ln(1/\delta))}{n}} + \frac{2 (k + \ln(1/\delta))}{n} \right)^2\mid \rA \in K_h(S)\right]\Pr_{\rA}[\rA \in K_h(S)] &=\\
\E_{\rA,\rt}\left[ \left(\sqrt{\frac{8\Loss_{\rA \cD}^{\gamma_i/2}(w)(k + \ln(1/\delta))}{n}} + \frac{2 (k + \ln(1/\delta))}{n} \right)^2\right].
\end{align*}
For positive constants $c_0,c_1,c_2$, we have that the function $f(t)=(\sqrt{c_0 t + c_1} + c_2)^2$ is concave for $t \geq 0$. To see this, we compute its derivative 
\[
f'(t) = 2(\sqrt{c_0 t + c_1} + c_2) \cdot \frac{c_0}{2\sqrt{c_0 t + c_1}} = c_0 + \frac{c_0 c_2}{\sqrt{c_0 t + c_1}},
\]
and its second derivative
\begin{align*}
f''(t) &= \frac{-c_0^2 c_2}{2 (c_0 t + c_1)^{3/2}}.
\end{align*}
This is a negative function for $t \geq 0$. We thus use Jensen's inequality to conclude
\begin{align*}
\E_{\rA,\rt}\left[ \left(\sqrt{\frac{8\Loss_{\rA \cD}^{\gamma_i/2}(w)(k + \ln(1/\delta))}{n}} + \frac{2 (k + \ln(1/\delta))}{n} \right)^2\right] &\leq \\
 \left(\sqrt{\frac{8\E_{\rA,\rt}\left[\Loss_{\rA \cD}^{\gamma_i/2}(w)\right] (k + \ln(1/\delta))}{n}} + \frac{2 (k + \ln(1/\delta))}{n} \right)^2.
\end{align*}
Combining it all, we have thus shown
\begin{align*}
\left|\E_{\rA,\rt} [\Loss^{\gamma_i/2}_{\rA \cD}(h_{\rA,\rt}(w)) - \Loss^{\gamma_i/2}_{\rA S}(h_{\rA,\rt}(w))]\right| &\leq \\
\sqrt{2^{33} \cdot \sqrt{\E_{\rA,\rt}[\|h_{\rA,\rt}(w)\|_2^4]}} \cdot \sqrt{\left(\sqrt{\frac{8\E_{\rA,\rt}\left[\Loss_{\rA \cD}^{\gamma_i/2}(w)\right] (k + \ln(1/\delta))}{n}} + \frac{2 (k + \ln(1/\delta))}{n} \right)^2} &\leq \\
2^{17} \cdot \E_{\rA,\rt}[\|h_{\rA,\rt}(w)\|_2^4]^{1/4} \cdot \left(\sqrt{\frac{8\E_{\rA,\rt}\left[\Loss_{\rA \cD}^{\gamma_i/2}(w)\right] (k + \ln(1/\delta))}{n}} + \frac{2 (k + \ln(1/\delta))}{n} \right).
\end{align*}
We now bound $\E_{\rA,\rt}[\|h_{\rA,\rt}(w)\|_2^4]$ as follows
\begin{align*}
\E_{\rA,\rt}[\|h_{\rA,\rt}(w)\|_2^4] &= \\
\E_{\rA,\rt}[\|\rA w + (h_{\rA,\rt}(w)-\rA w)\|_2^4] &\leq \\
\E_{\rA,\rt}[\left(\|\rA w\|_2 + \|h_{\rA,\rt}(w)-\rA w\|_2\right)^4] &\leq\\
\E_{\rA,\rt}\left[\left(\|\rA w\|_2 + \sqrt{k (10\sqrt{k})^{-2}}\right)^4\right] &=\\
\E_{\rA,\rt}\left[\left(\|\rA w\|_2 + 1/10\right)^4\right] &=\\
\sum_{b=0}^4 \binom{4}{b} \E_{\rA,\rt}[\|\rA w\|_2^b] 10^{-(4-b)}.
\end{align*}
Recalling that $\|\rA w\|_2^2 \sim (1/k)\chi_k^2$, we have from the moments of the chi-square distribution that for even $k \geq 4$:
\[
\E_{\rA,\rt}[\|\rA w\|_2^b] \leq \E_{\rA,\rt}[\|\rA w\|_2^4] =k^{-2}\E_{\rA,\rt}[(k\|\rA w\|_2^2)^2] = k^{-2} 2^2 \frac{(2 + k/2)!}{(k/2)!} \leq 4.
\]
Hence
\begin{align*}
\E_{\rA,\rt}[\|h_{\rA,\rt}(w)\|_2^4] \leq
\sum_{b=0}^4 \binom{4}{b} 4 \cdot 10^{-(4-b)} \leq (4 + 1/10)^4 < 5^4.
\end{align*}
We thus have
\begin{align*}
\left|\E_{\rA,\rt} [\Loss^{\gamma_i/2}_{\rA \cD}(h_{\rA,\rt}(w)) - \Loss^{\gamma_i/2}_{\rA S}(h_{\rA,\rt}(w))]\right| &\leq \\
2^{20} \cdot \left(\sqrt{\frac{8\E_{\rA,\rt}\left[\Loss_{\rA \cD}^{\gamma_i/2}(w)\right] (k + \ln(1/\delta))}{n}} + \frac{2 (k + \ln(1/\delta))}{n} \right).
\end{align*}
Finally, we exploit that for any $w \in \subH(\Gamma_i, L_j)$, we have by definition that $\Loss_\cD^{(3/4)\gamma_{i}}(w)\leq \ell_{j+1}$. Thus for any such $w$, we have
\begin{align*}
\E_{\rA,\rt}[\Loss^{\gamma_i/2}_{\rA \cD}(w)] &=\E_{\rA,\rt}[\Pr_{(\rx,\ry)\sim \cD}[\ry \langle h_{\rA,\rt}(w), \rA\rx \rangle \leq \gamma_i/2]] \\
&=
\E_{(\rx,\ry)\sim \cD}[\Pr_{\rA,\rt}[\ry \langle h_{\rA,\rt}(w), \rA\rx \rangle \leq \gamma_i/2]]
\\&\leq
\Pr_{(\rx,\ry)\sim \cD}[\ry\langle w, \rx \rangle \leq (3/4)\gamma_{i}] \\
&+ \E_{(\rx,\ry)\sim \cD}[\Pr_{\rA,\rt}[\ry \langle h_{\rA,\rt}(w), \rA\rx \rangle \leq \gamma_i/2] \mid \ry\langle w, \rx \rangle > (3/4)\gamma_{i}] \\
&\leq \Loss_{\cD}^{(3/4)\gamma_{i}}(w)  + \sup_{\mu > (3/4)\gamma_i}[\Pr_{\rA,\rt}[\langle h_{\rA,\rt}(w), \rA x \rangle \leq \gamma_i/2 \mid y\ipr{w,x}=\mu].
\end{align*}
Using Lemma~\ref{lem:concdiscretize} and that $\Loss^{(3/4)\gamma_i}_\cD(w) \in L_j$ by definition of $\subH(\Gamma_i,L_j)$, there is a constant $c>0$ such that this is bounded by
\begin{align*}
&\leq \Loss_{\cD}^{(3/4)\gamma_{i}}(w) + c\exp(-k(\gamma_{i}/4)^2/c)\\
&\leq \ell_{j+1} + c \exp(-k \gamma_{i+1}^2/(16 c)).
\end{align*}
We have thus reached the conclusion that there is a constant $c>0$, such that with probability at least $1-\delta$ over $\rS \sim \cD^n$, it holds that
\begin{align*}
    \sup_{w \in \subH(\Gamma_i,L_j)} \left|\E_{\rA,\rt} [\Loss^{\gamma_i/2}_{\rA \cD}(h_{\rA,\rt}(w)) - \Loss^{\gamma_i/2}_{\rA S}(h_{\rA,\rt}(w))]\right| &\leq\\ c \cdot \left(\sqrt{\frac{(\ell_{j+1} + \exp(-k \gamma_{i+1}^2/c)) (k + \ln(1/\delta))}{n}} + \frac{k + \ln(1/\delta)}{n} \right).
\end{align*}
This completes the proof of Lemma~\ref{lem:supW}.

%% file: factortwo.tex
\section{Within Constant Factors}
\label{sec:withinconstant}
In this section we prove
\begin{customlem}{\ref{lem:subgoal2}}
There is a constant $c>1$, such that for any $0 < \delta < 1$ and any $\Gamma_i = (\gamma_i, \gamma_{i+1}]$, it holds with probability at least $1-\delta$ over a random sample $\rS \sim \cD^n$ that
\begin{align*}
\forall w \in \finalH : \Loss_{\rS}^{\gamma_i}(w) \geq \frac{\Loss_{ \cD}^{(3/4)\gamma_i}(w)}{4} - c \left(\frac{\ln(e\gamma_{i+1}^2 n)}{\gamma_{i+1}^2 n} - \frac{\ln(e/\delta)}{n}\right).
\end{align*}
\end{customlem}
The proof follows mostly the ideas in~\cite{SVMbest} that were outlined in the proof overview in Section~\ref{sec:overview}. 

\begin{proof}
Let $k \geq 1$ be a parameter to be determined and consider the random construction of $\rA$ and $\rt$ as defined in Section~\ref{sec:mainargs}. Let $\Disc_a$ be defined as in Section~\ref{sec:meetinmid}, i.e.\ $\Disc_a$ contains all vectors in $2^a \cdot 4 \Ball_2^k$. We say that a matrix $A$ in the support of $\rA$ and a training set $S$ is $\alpha$-\emph{unusual}, if there is a vector $w \in \Disc_0$ such that
\[
\Loss_{A S}^{(7/8)\gamma_i}(w) < \frac{\Loss_{A \cD}^{(7/8)\gamma_i}(w)}{2} - \frac{2^{11} k + \ln(1/\alpha)}{n}.
\]
For a fixed matrix $A$ and vector $w \in \Net_0$, we have by Bernstein's inequality and $\E_{\rS}[\Loss_{A S}^{(7/8)\gamma_i}(w)]= \Loss_{A \cD}^{(7/8)\gamma_i}(w)$ that
\begin{align*}
    \Pr_{\rS}\left[\left|\Loss_{A \rS}^{(7/8)\gamma_i}(w) -\Loss_{A \cD}^{(7/8)\gamma_i}(w)\right|>t/n \right] < \exp\left(- \frac{\frac{1}{2} t^2}{n\Loss_{A \cD}^{(7/8)\gamma_i}(w) + \frac{1}{3}t}\right).
\end{align*}
Setting
\[
t = n \cdot \left( \frac{\Loss_{A \cD}^{(7/8)\gamma_i}(w)}{2} + Z\right)
\]
with $Z=16 \ln(1/\alpha)/n$ gives
\begin{align*}
    \Pr_{\rS}\Bigg[\bigg|\Loss_{A \rS}^{(7/8)\gamma_i}(w) -\Loss&_{A \cD}^{(7/8)\gamma_i}(w)\bigg|>\left( \frac{\Loss_{A \cD}^{(7/8)\gamma_i}(w)}{2} + Z\right) \Bigg]\\ 
    &< \exp\left(- \frac{\frac{n^2}{2} \left( \frac{\Loss_{A \cD}^{(7/8)\gamma_i}(w)}{2} + Z\right)^2}{n\Loss_{A \cD}^{(7/8)\gamma_i}(w) + \frac{n}{3}\left( \frac{\Loss_{A \cD}^{(7/8)\gamma_i}(w)}{2} + Z\right)}\right) \\
    &\leq \exp\left(- \frac{\frac{n^2}{8} \max\{\Loss_{A \cD}^{(7/8)\gamma_i}(w), Z\}^2}{2n \max\{\Loss_{A \cD}^{(7/8)\gamma_i}(w), Z\}}\right) \\
    &\leq \exp\left(-\frac{n Z}{16} \right) \\
    &= \alpha.
\end{align*}
A union bound over all $w \in \Disc_0$ with $\alpha' = \alpha/e^{2^7 k}$ gives that a fixed matrix $A$ is $\alpha$-unusual for $\rS \sim \cD^n$ with probability at most
\[
|\Disc_0| \frac{\alpha}{e^{2^7 k}} < \alpha.
\]
Now call a training set $S$ $\alpha$-representative if $\rA$ is $\alpha$-unusual for $S$ with probability less than $1/4$. By Markov's inequality, we have
\begin{align*}
\Pr_{\rS}[\Pr_{\rA}[(\rS,\rA) \textrm{ is $\alpha$-unusual}] \geq 1/4] &\leq \frac{\E_{\rS}[\Pr_{\rA}[(\rS,\rA) \textrm{ is $\alpha$-unusual}]]}{1/4} \\
&= 4 \cdot \E_{\rA}[\Pr_{\rS}[(\rS,\rA) \textrm{ is $\alpha$-unusual}]]\\
&\leq 4 \alpha.
\end{align*}
Thus
\begin{align}
\Pr_\rS[\rS \textrm{ is $\alpha$-representative}] \geq 1-4 \alpha.\label{eq:oftenrep}
\end{align}
We claim that if the training set $S$ is $\delta$-representative, then it holds for all $w \in \finalH$ that
\[
\Loss^\gamma_{S}(w) \geq \frac{\Loss_{ \cD}^{(3/4)\gamma_i}(w)}{4} - \frac{2^{11} k + \ln(4/\delta)}{n} -30 \exp(-k \gamma_{i+1}^2/2^{14}).
\]
To see this, consider an arbitrary such $S$ and a $w \in \finalH$. Sample $\rA$ and $\rt$ as in the previous section. Call $\rA, \rt$ \emph{good} for $w$ if it satisfies both $\|h_{\rA,\rt}(w)\|_2 \leq 4$ and $\Loss^{(7/8)\gamma_i}_{\rA \cD}(h_{\rA,\rt}(w)) \geq \Loss^{(3/4)\gamma_i}_{\cD}(w) - 25 \exp(-k \gamma_{i+1}^2/2^{14})$. For ease of notation, let $G_w$ denote the set of $(A,t)$ that are good for $w$. Similarly, let $U_S$ denote the set of $A$ where $A$ is $\delta$-unusual for $S$.

For all $w \in \finalH$, $\gamma \in \Gamma_i$, $A$ and $t$, we have that
\begin{align*}
    \Loss_{S}^\gamma(w) &\geq \Loss^{(7/8)\gamma_i}_{A S}(h_{A,\rt}(w)) - \Pr_{(\rx,\ry)\sim S}[\ry \langle w, \rx \rangle > \gamma \wedge \ry \langle h_{A,\rt}(w), A \rx \rangle \leq (7/8)\gamma_i].
\end{align*}
Thus
\begin{align}
    \Loss_{S}^\gamma(w) &\geq \E_{\rA,\rt}[\Loss^{(7/8)\gamma_i}_{\rA S}(h_{\rA,\rt}(w)) - \Pr_{(\rx,\ry)\sim S}[\ry \langle w, \rx \rangle > \gamma \wedge \ry \langle h_{\rA,\rt}(w), \rA \rx \rangle \leq (7/8)\gamma_i]] \nonumber\\
    &\geq \E_{\rA,\rt}[\Loss^{(7/8)\gamma_i}_{\rA S}(h_{\rA,\rt}(w)) \mid (\rA,\rt) \in G_w \wedge \rA \notin U_S] \Pr_{\rA,\rt}[(\rA,\rt) \in G_w \wedge \rA \notin U_S] \label{eq:condGood}\\
    &-\E_{\rA,\rt}[\Pr_{(\rx,\ry)\sim S}[\ry \langle w, \rx \rangle > \gamma \wedge \ry \langle h_{\rA,\rt}(w), \rA \rx \rangle \leq (7/8)\gamma_i]]\label{eq:randroundoff}.
\end{align}
For the term~\eqref{eq:condGood}, we observe that conditioned on $(\rA,\rt) \in G_w$, we have that $h_{\rA,\rt}(w) \in \Disc_0$ since $\|h_{\rA,\rt}(w)\|_2 \leq 4$. Secondly, when $\rA \notin U_S$, this implies by the definition of $\delta$-unusual that
\[
\Loss_{A S}^{(7/8)\gamma_i}(h_{\rA,\rt}(w)) \geq \frac{\Loss_{A \cD}^{(7/8)\gamma_i}(h_{\rA,\rt}(w))}{2} - \frac{2^{11} k + \ln(1/\delta)}{n}.
\]
Hence
\begin{align}
    &\E_{\rA,\rt}[\Loss^{(7/8)\gamma_i}_{\rA S}(h_{\rA,\rt}(w)) \mid (\rA,\rt) \in G_w \wedge \rA \notin U_S] \Pr_{\rA,\rt}[(\rA,\rt) \in G_w \wedge \rA \notin U_S] \geq \nonumber\\
    &\E_{\rA,\rt}\Bigg[\frac{\Loss_{A \cD}^{(7/8)\gamma_i}(h_{\rA,\rt}(w))}{2} \Bigg|(\rA,\rt) \in G_w \wedge \rA \notin U_S\Bigg]\Pr_{\rA,\rt}[(\rA,\rt)\in G_w \wedge \rA \notin U_S] - \frac{2^{11} k + \ln(1/\delta)}{n}.\label{eq:adbound}
\end{align}
Using again that $(\rA, \rt) \in G_w$, we have that~\eqref{eq:adbound} is at least
\begin{align}
    \frac{\Loss_{ \cD}^{(3/4)\gamma_i}(w)}{2} \Pr_{\rA,\rt}[(\rA,\rt)\in G_w \wedge \rA \notin U_S] - \frac{2^{11} k + \ln(1/\delta)}{n} -25 \exp(-k \gamma_{i+1}^2/2^{14}). \label{eq:lastprop}
\end{align}
We now bound $\Pr[(\rA,\rt) \in G_w]$  and $\Pr[\rA \notin U_S]$. For this, we recall that $\|\rA w\|_2^2 \sim (1/k)\chi_2^k$. Thus $\E[\|\rA w\|_2^2] =1$ and by Markov's, we get $\Pr[\|\rA w\|_2^2 \geq 9] \leq 1/9$. Conditioned on $\|\rA w\|_2^2 < 9$, we have $\|h_{\rA,\rt}(w)\|_2 \leq \|\rA w\|_2 + \|h_{\rA,\rt}(w)-\rA w\|_2 \leq \sqrt{9} + \sqrt{k (10 \sqrt{k})^{-2}} < 4$. Next observe that
\begin{align*}
    \Loss_{\rA \cD}^{(7/8)\gamma_i}(h_{\rA,\rt}(w)) &\geq \Loss_{\cD}^{(3/4)\gamma_i}(w) - \Pr_{(\rx, \ry)\sim \cD}[\ry \langle w, \rx \rangle \leq (3/4)\gamma_i \wedge \ry \langle h_{\rA,\rt}(w), \rA \rx \rangle > (7/8)\gamma_i].
\end{align*}
We have by Lemma~\ref{lem:concdiscretize} that there is a constant $c>0$ so that
\begin{align*}
    \E_{\rA,\rt}[\Pr_{(\rx, \ry)\sim \cD}[\ry \langle w, \rx \rangle \leq (3/4)\gamma_i \wedge \ry \langle h_{\rA,\rt}(w), \rA \rx \rangle > (7/8)\gamma_i]] &= \\
    \E_{(\rx, \ry)\sim \cD}[\Pr_{\rA,\rt}[\ry \langle w, \rx \rangle \leq (3/4)\gamma_i \wedge \ry \langle h_{\rA,\rt}(w), \rA \rx \rangle > (7/8)\gamma_i]] &\leq \\
    \sup_{x \in \finalX : \langle w, x\rangle \leq (3/4)\gamma_i}\Pr_{\rA,\rt}[ \langle h_{\rA,\rt}(w), \rA x \rangle > (7/8)\gamma_i] &\leq \\
    c\exp(-k(\gamma_i/8)^2/c) &\leq \\
    c\exp(-k \gamma_{i+1}^2/(2^{8}c)).
\end{align*}
Thus by Markov's inequality, we conclude
\begin{align*}
\Pr_{\rA,\rt}[\Loss_{\rA \cD}^{(7/8)\gamma_i}(h_{\rA,\rt}(w)) <\Loss_{\cD}^{(3/4)\gamma_i}(w)- 5c \exp(-k \gamma_{i+1}^2/(2^{8}c))] &\leq \\
    \Pr_{\rA,\rt}[\Pr_{(\rx, \ry)\sim \cD}[\ry \langle w, \rx \rangle \leq (3/4)\gamma_i \wedge \ry \langle h_{\rA,\rt}(w), \rA \rx \rangle > (7/8)\gamma_i] > 5 c \exp(-k \gamma_{i+1}^2/(2^{8}c))] &< 1/5.
\end{align*}
Finally, since we assumed $S$ is $\delta$-representative, we have $\Pr_{\rA}[\rA \in U_S] \leq 1/4$ by definition of $\delta$-representative. We conclude by a union bound that
\begin{align*}
\Pr_{\rA,\rt}[(\rA,\rt)\in G_w \wedge \rA \notin U_S] &\geq 1-1/9-1/5-1/4\geq 1/2.
\end{align*}
In summary, we have shown that~\eqref{eq:lastprop} is at least
\[
\frac{\Loss_{ \cD}^{(3/4)\gamma_i}(w)}{2} \cdot \frac{1}{2} - \frac{2^{11} k + \ln(1/\delta)}{n} -5 c \exp(-k \gamma_{i+1}^2/(2^{8}c)).
\]
Recalling that~\eqref{eq:condGood}~$\geq$~\eqref{eq:lastprop} gives
\begin{align*}
\E_{\rA,\rt}[\Loss^{(7/8)\gamma_i}_{\rA S}(h_{\rA,\rt}(w)) \mid (\rA,\rt) \in G_w \wedge \rA \notin U_S] \Pr_{\rA,\rt}[(\rA,\rt) \in G_w \wedge \rA \notin U_S] &\geq \\
\frac{\Loss_{ \cD}^{(3/4)\gamma_i}(w)}{4} - \frac{2^{11} k + \ln(1/\delta)}{n} -5c \exp(-k \gamma_{i+1}^2/(2^{8}c)).
\end{align*}
The term~\eqref{eq:randroundoff} can be bounded using Lemma~\ref{lem:concdiscretize} by
\begin{align*}
\E_{\rA,\rt}[\Pr_{(\rx,\ry)\sim S}[\ry \langle w, \rx \rangle > \gamma \wedge \ry \langle h_{\rA,\rt}(w), \rA \rx \rangle \leq (7/8)\gamma_i]] &= \\
\E_{(\rx,\ry)\sim S}[\Pr_{\rA,\rt}[\ry \langle w, \rx \rangle > \gamma \wedge \ry \langle h_{\rA,\rt}(w), \rA \rx \rangle \leq (7/8)\gamma_i]] &\leq \\
\sup_{x \in \finalX : \langle w, x\rangle > \gamma }\Pr_{\rA,\rt}[\langle h_{\rA,\rt}(w), \rA x \rangle \leq (7/8)\gamma_i] &\leq\\
c\exp(-k(\gamma - (7/8)\gamma_i)^2/c) &\leq \\
c\exp(-k\gamma_i^2/(64 c))&\leq \\
c\exp(-k\gamma_{i+1}^2/(2^{8}c)).
\end{align*}
In summary, we have shown that for $(\delta/4)$-representative $S$, it holds for all $w \in \finalH$ that
\begin{align*}
    \Loss_S^\gamma(w) \geq \frac{\Loss_{ \cD}^{(3/4)\gamma_i}(w)}{4} - \frac{2^{11} k + \ln(4/\delta)}{n} -6c \exp(-k \gamma_{i+1}^2/(2^{8}c)).
\end{align*}
We finally conclude from~\eqref{eq:oftenrep} that with probability at least $1-\delta$ over $\rS$, it holds for all $w \in \finalH$ that
\begin{align*}
    \Loss_{\rS}^\gamma(w) \geq \frac{\Loss_{ \cD}^{(3/4)\gamma_i}(w)}{4} - \frac{2^{11} k + \ln(4/\delta)}{n} -6c \exp(-k \gamma_{i+1}^2/(2^{8}c)).
\end{align*}
Picking $k=2^{8}c \gamma_{i+1}^{-2} \ln(\gamma^2_{i+1} n)$ finally results in
\begin{align*}
    \Loss_{\rS}^\gamma(w) \geq \frac{\Loss_{ \cD}^{(3/4)\gamma_i}(w)}{4} - \frac{2^{20}c \ln(\gamma_{i+1}^2 n)}{\gamma_{i+1}^2 n} - \frac{2\ln(e/\delta)}{n}.
\end{align*}
This completes the proof.
\end{proof}

%% file: acknow.tex
\section*{Acknowledgment}
The authors would like to thank Clement Svendsen for valuable measure theoretic insight. 

Kasper Green Larsen is co-funded by a DFF Sapere Aude Research Leader Grant No. 9064-00068B by the Independent Research Fund Denmark and co-funded by the European Union (ERC, TUCLA, 101125203). Natascha Schalburg is funded by the European Union (ERC, TUCLA, 101125203). Views and opinions expressed are however those of the author(s) only and do not necessarily reflect those of the European Union or the European Research Council. Neither the European Union nor the granting authority can be held responsible for them.

%% file: auxiliary.tex
\section{Auxiliary Results}
\label{sec:aux}
In this section, we prove a number of auxiliary results used throughout the paper. For this, we need the following concentration inequality:

\begin{theorem}[\cite{Wainwright_2019}, example 2.11]
\label{thm:chibound}
Let $Y \sim \chi^2_k$, then for any $x \in (0,1)$ it holds that
\[
\Pr\left[\left|\frac{Y}{k} - 1\right| \ge x\right] \le 2\exp(-k x^2/8).
\]
\end{theorem}

\begin{customclm}{\ref{clm:union}}
    For any $0 < \delta < 1$, it holds with probability $1-\delta$ over $\rS \sim \cD^n$ that~\eqref{eq:subgoal} and~\eqref{eq:subgoal2} simultaneously hold for all $(\Gamma_i,L_j)$ and $\Gamma_i$, with slightly different constants $c$.
\end{customclm}
\begin{proof}
Let $(\gamma_i, \gamma_{i+1}]$ be such that $\gamma_{i+1} := 2^{i}n^{-1/2}$. Similarly, let $(\ell_j, \ell_{j+1}]$ be such that $\ell_{j+1} := 2^j n^{-1}$. Do a union bound over all $(\Gamma_i, L_j)$ for $i=1,\dots,\lg_2(c_\gamma n^{1/2})$ and $j=0,\dots,\lg_2 n$ with $\delta_{i,j} := (\delta/e)^3 \exp(-\gamma_{i+1}^{-2} \ln(e/\ell_{j+1}))$ in~\eqref{eq:subgoal}. We see that
\begin{align*}
    \sum_{i=1}^{\lg_2(c_\gamma n^{1/2})} \sum_{j=0}^{\lg_2 n} \delta_{i,j} &= \sum_{i=1}^{\lg_2(c_\gamma n^{1/2})} \sum_{j=0}^{\lg_2 n} (\delta/e)^3 \exp(-\gamma_{i+1}^{-2} \ln(e/\ell_{j+1})) \\
    &= \sum_{i=1}^{\lg_2(c_\gamma n^{1/2})} \sum_{j=0}^{\lg_2 n}  (\delta/e)^3 \exp(- 2^{-2i}n \ln(e n 2^{-j})) \\
    &= \sum_{i=1}^{\lg_2(c_\gamma n^{1/2})} \sum_{j=0}^{\lg_2 n}  (\delta/e)^3 (e n 2^{-j})^{-2^{-2i}n} \\
\end{align*}
Doing the substitutions $j \gets \lg_2 n-j$ and $i \gets \lg_2(c_\gamma n^{1/2}) +1 - i$, this equals
\begin{align*}
    &= \sum_{i=1}^{\lg_2(c_\gamma n^{1/2})} \sum_{j=0}^{\lg_2 n}  (\delta/e)^3 (e 2^{j})^{-2^{2i-2}c_\gamma^{-2}} \\
    &\leq \sum_{i=1}^{\lg_2(c_\gamma n^{1/2})} \sum_{j=0}^{\lg_2 n}  (\delta/e)^3 e^{-2^{2i-2}} 2^{-j}\\ 
    &\leq \sum_{i=1}^{\lg_2(c_\gamma n^{1/2})} 2(\delta/e)^3 e^{-2^{2i-2}}\ \\
    &\leq \delta/2.
\end{align*}
Similarly, do a union bound over all $\Gamma_i$ with $\delta_i := (\delta/e)^3 \exp(-\gamma_{i+1}^{-2} \ln(e \gamma^2_{i+1} n))$ in~\eqref{eq:subgoal2}. We have
\begin{align*}
\sum_{i=1}^{\lg_2(c_\gamma n^{1/2})} \delta_i &= \sum_{i=1}^{\lg_2(c_\gamma n^{1/2})} (\delta/e)^3 \exp(-\gamma_{i+1}^{-2} \ln(e \gamma^2_{i+1} n)) \\
&\leq\sum_{i=1}^{\lg_2(c_\gamma n^{1/2})}(\delta/e)^3 \exp(-\ln(e \gamma^2_{i+1} n))\\
&= \sum_{i=1}^{\lg_2(c_\gamma n^{1/2})}(\delta/e)^3 \frac{1}{e \gamma_{i+1}^2 n} \\
&= \sum_{i=1}^{\lg_2(c_\gamma n^{1/2})}(\delta/e)^3 \frac{n}{e 2^{2i} n} \\
&\leq (\delta/e)^3 \\
&\leq \delta/2.
\end{align*}
We thus have that with probability at least $1-\delta$ that for all $(\Gamma_i, L_j)$, we have
\begin{align*}
\sup_{w \in \subH(\Gamma_i, L_j), \gamma \in \Gamma_i} \left|\Loss_\cD(w) - \Loss^{\gamma}_\rS(w) \right| &\leq \nonumber\\
c \left(\sqrt{\ell_{j+1}\left( \frac{ \ln(e/\ell_{j+1})}{\gamma_{i+1}^2 n} + \frac{\ln(e/\delta_{i,j})}{n} \right)} + \frac{\ln(e/\ell_{j+1})}{\gamma_{i+1}^2 n} + \frac{\ln(e/\delta_{i,j})}{n} \right) &=\\
c \left(\sqrt{\ell_{j+1}\left( \frac{ \ln(e/\ell_{j+1})}{\gamma_{i+1}^2 n} + \frac{\ln(e/\delta_{i,j})}{n} \right)} + 2\cdot \frac{\ln(e/\ell_{j+1})}{\gamma_{i+1}^2 n} + 3 \cdot \frac{\ln(e/\delta)}{n} \right).
\end{align*}
and for all $\Gamma_i$, we have
\begin{align*}
\inf_{w \in \finalH} \Loss_{\rS}^{\gamma_i}(w) &\geq \frac{\Loss_{ \cD}^{(3/4)\gamma_i}(w)}{4} - c \left(\frac{\ln(e\gamma_{i+1}^2 n)}{\gamma_{i+1}^2 n} - \frac{\ln(e/\delta_i)}{n}\right)\\ &=
\frac{\Loss_{ \cD}^{(3/4)\gamma_i}(w)}{4} - c \left(2 \cdot \frac{\ln(e\gamma_{i+1}^2 n)}{\gamma_{i+1}^2 n} - 3 \cdot \frac{\ln(e/\delta)}{n}\right).
\end{align*}
\end{proof}

\begin{customclm}{\ref{clm:combine}}
    For any $0 < \delta < 1$ and training set $S$, if~\eqref{eq:subgoal} and~\eqref{eq:subgoal2} hold simultaneously for all $(\Gamma_i, L_j)$ and $\Gamma_i$, then~\eqref{eq:maingoal} holds for all $\gamma \in (n^{-1/2}, c_\gamma]$ and all $w \in \finalH$ for large enough constant $c>1$ in~\eqref{eq:maingoal}.
\end{customclm}
\begin{proof}
Let $0 < \delta < 1$ and assume as in the claim that~\eqref{eq:subgoal} and~\eqref{eq:subgoal2} holds for all $(\Gamma_i, L_j)$ and $\Gamma_i$. Now consider an arbitrary $\gamma \in (n^{-1/2},c_\gamma]$ and $w \in \finalH$. Let $i$ and $j$ be such that $\gamma \in (\gamma_i, \gamma_{i+1}]$ and $\Loss^{(3/4)\gamma_{i}}_\cD(w) \in (\ell_j, \ell_{j+1}]$ with $\gamma_{i+1}=2^i n^{-1/2}$ and $\ell_{j+1} = 2^j n^{-1}$. We consider two cases. Let $c_{\ref{lem:subgoal2}}>1$ be the constant in Lemma~\ref{lem:subgoal2}. First, if 
\[
\Loss_\cD^{(3/4)\gamma_{i}}(w) \leq 16 \cdot c_{\ref{lem:subgoal2}} \cdot \left(\frac{\ln(e\gamma_{i+1}^2 n)}{\gamma_{i+1}^2 n} + \frac{\ln(e/\delta)}{n} \right),
\]
then since $\Loss_\cD(w) \leq \Loss_\cD^{(3/4)\gamma_{i}}(w)$ and $\gamma \leq \gamma_{i+1}$ (using that $\ln(e \gamma^2 n)/(\gamma^2 n)$ is decreasing in $\gamma$ for $\gamma \geq n^{-1/2}$), we have already shown~\eqref{eq:maingoal} for sufficiently large constant $c$ in~\eqref{eq:maingoal}. So assume this is not the case. Our goal is to show that $\ell_{j+1}$ and $\Loss_S^\gamma(w)$ are within constant factors of each other so that we may replace occurrences of $\ell_{j+1}$ by $\Loss_S^\gamma(w)$ in~\eqref{eq:subgoal}. We first see that our assumption implies
\begin{align}
\ell_{j+1} \geq \Loss_\cD^{(3/4)\gamma_{i}}(w) \geq 16 \cdot c_{\ref{lem:subgoal2}} \cdot \left(\frac{\ln(e\gamma_{i+1}^2 n)}{\gamma_{i+1}^2 n} + \frac{\ln(e/\delta)}{n} \right) \geq   \frac{1}{\gamma_{i+1}^2 n} > n^{-1}.\label{eq:ljrelate}
\end{align}
This also implies $j \neq 0$ and hence $\ell_{j+1}=2\ell_j$ and therefore $\ell_{j+1} \leq 2 \Loss_{\cD}^{(3/4)\gamma_{i}}(w)$. Letting $c_{\ref{lem:subgoal}}$ be the constant in Lemma~\ref{lem:subgoal}, we get from~\eqref{eq:subgoal} and~\eqref{eq:ljrelate}, that
\begin{align*}
    \Loss^\gamma_S(w) &\leq \Loss_\cD(w) + c_{\ref{lem:subgoal}} \cdot \left(\sqrt{\ell_{j+1}\left( \frac{ \ln(e/\ell_{j+1})}{\gamma_{i+1}^2 n} + \frac{\ln(e/\delta)}{n} \right)} + \frac{\ln(e/\ell_{j+1})}{\gamma_{i+1}^2 n} + \frac{\ln(e/\delta)}{n} \right) \\
    \leq \ & \Loss^{(3/4)\gamma_{i}}_\cD(w) + c_{\ref{lem:subgoal}} \cdot \left(\sqrt{2 \Loss^{(3/4)\gamma_{i}}_\cD(w)\left( \frac{ \ln(e\gamma_{i+1}^2 n)}{\gamma_{i+1}^2 n} + \frac{\ln(e/\delta)}{n} \right)} + \frac{\ln(e \gamma_{i+1}^2 n)}{\gamma_{i+1}^2 n} + \frac{\ln(e/\delta)}{n} \right)\\
    \leq \ & \Loss^{(3/4)\gamma_{i}}_\cD(w) + c_{\ref{lem:subgoal}} \cdot \left(\sqrt{2 \Loss^{(3/4)\gamma_{i}}_\cD(w)\Loss^{(3/4)\gamma_{i}}_\cD(w)} + \Loss^{(3/4)\gamma_{i}}_\cD(w) \right)\\
    \leq \ & 3\cdot c_{\ref{lem:subgoal}} \cdot \Loss^{(3/4)\gamma_{i}}_\cD(w).
\end{align*}
We thus also have $\ell_{j+1} \geq \Loss^{(3/4)\gamma_{i}}_\cD(w) \geq (3\cdot c_{\ref{lem:subgoal}})^{-1} \Loss_S^\gamma(w)$. Inserting this and~\eqref{eq:ljrelate} in~\eqref{eq:subgoal} gives
\begin{align}
\Loss_{\cD}(w) \leq \Loss_{S}^\gamma(w) + c_{\ref{lem:subgoal}} \left( \sqrt{\ell_{j+1}\left( \frac{ \ln(3e c_{\ref{lem:subgoal}}/\Loss_{S}^\gamma(w))}{\gamma_{i+1}^2 n} + \frac{\ln(e/\delta)}{n} \right)} + \frac{\ln(e\gamma_{i+1}^2 n)}{\gamma_{i+1}^2 n} + \frac{\ln(e/\delta)}{n} \right).\label{eq:almostdone}
\end{align}
Finally from~\eqref{eq:subgoal2} and $\gamma \geq \gamma_i$, we have
\begin{align*}
\Loss_{S}^\gamma(w) &\geq \Loss_{S}^{\gamma_i}(w) \\
&\geq \frac{\Loss_{\cD}^{(3/4)\gamma_i}(w)}{4} - c_{\ref{lem:subgoal2}}\left( \frac{\ln(e\gamma_{i+1}^2 n)}{\gamma_{i+1}^2 n} - \frac{\ln(e/\delta)}{n}\right)\\
&\geq \frac{\ell_{j+1}}{8} - c_{\ref{lem:subgoal2}}\left( \frac{\ln(e\gamma_{i+1}^2 n)}{\gamma_{i+1}^2 n} - \frac{\ln(e/\delta)}{n}\right).
\end{align*}
From~\eqref{eq:ljrelate}, this is at least $\ell_{j+1}/16$ and thus $\ell_{j+1} \leq 16 \Loss^{\gamma}_S(w)$. Inserting this in~\eqref{eq:almostdone} finally gives us
\begin{align*}
\Loss_{\cD}(w) \leq \Loss_{S}^\gamma(w) + c_{\ref{lem:subgoal}} \left( \sqrt{16 \Loss^\gamma_S(w) \left( \frac{ \ln(2e c_{\ref{lem:subgoal}}/\Loss_{S}^\gamma(w))}{\gamma_{i+1}^2 n} + \frac{\ln(e/\delta)}{n} \right)} + \frac{\ln(e\gamma_{i+1}^2 n)}{\gamma_{i+1}^2 n} + \frac{\ln(e/\delta)}{n} \right).
\end{align*}
Since $\gamma \leq \gamma_{i+1}$, this completes the proof of Claim~\ref{clm:combine} for sufficiently large $c>0$ in~\eqref{eq:maingoal}.
\end{proof}

\begin{customlem}{\ref{lem:concdiscretize}}
    There is a constant $c>0$, such that for any integer $k \geq 1$, $w \in \finalH, x \in \finalX$ and any $\gamma \in (0,1]$, it holds that 
    $
    \Pr_{\rA,\rt}[|\langle h_{\rA,\rt}(w),\rA x\rangle - \langle w, x\rangle| > \gamma] < c\exp(-\gamma^2 k/c)
    $.    
\end{customlem}
\begin{proof} %(of Lemma~\ref{lem:concdiscretize})
We start by observing that $\|\rA w\|_2^2$, $\|\rA x\|_2^2$ and $\|\rA (w-x)\|_2^2/\|w-x\|_2^2$ are all $(1/k)\chi^2_k$ distributed. Using Theorem~\ref{thm:chibound} with $x=\gamma/3$, we have with probability at least $1-6 \exp(-k \gamma^2/72)$ that $\|\rA w\|_2^2 \in 1 \pm \gamma/3$, $\|\rA x\|_2^2 \in 1 \pm \gamma/3$ and $\|\rA(w-x)\|_2^2 \in \|w-x\|_2^2(1 \pm \gamma/3)$. By the polar identity, this implies
\begin{align*}
\langle \rA w, \rA w \rangle &= \frac{1}{4} \left(\|\rA w\|_2^2 + \|\rA x\|_2^2 - \|\rA(w-x)\|_2^2 \right) \\
&\in \frac{1}{4}\left(\|w\|_2^2 + \|x\|_2^2 - \|w-x\|_2^2 \right) \pm \frac{\gamma}{12}\left(\|w\|_2^2 + \|x\|_2^2 + \|w-x\|_2^2 \right) \\
&\subseteq \langle w, x \rangle \pm \frac{\gamma}{12}\left(1 + 1 + 4 \right) \\
&= \langle w, x \rangle \pm \frac{\gamma}{2}.
\end{align*}
Let us condition on an outcome $A$ of $\rA$ satisfying the above. We then observe that
\begin{align*}
\langle h_{A, \rt}(w) , A x\rangle &= \langle h_{A, \rt}(w) - A w, A x\rangle + \langle A w, A x\rangle.
\end{align*}
By the randomized rounding procedure, we have that each coordinate $i$ satisfies $\E_{\rt_i}[(h_{A,\rt}(w))_i] = (A w)_i$. Moreover, these coordinates are independent. Letting $\Delta_i = (h_{A,\rt}(w))_i - (A w)_i$, we then have that $\E[\Delta_i] = 0$ and that $\Delta_i$ lies in an interval of length $(10 \sqrt{k})^{-1}$. Hoeffding's inequality implies
\begin{align*}
\Pr_{\rt}[\left|\langle h_{A, \rt}(w) - A w, A x\rangle\right| > \gamma/2] &= \Pr_{\Delta_1,\dots,\Delta_k}\left[\left|\sum_{i=1}^k \Delta_i(A x)_i\right| > \gamma/2\right] \\
&< 2 \exp\left( -\frac{2 (\gamma/2)^2}{\sum_{i=1}^k (10 \sqrt{k})^{-2}(Ax)_i^2} \right) \\
&= 2 \exp\left( -\frac{50 \gamma^2 k}{\|A x\|_2^2} \right) \\
&\leq 2 \exp\left( - 25 \gamma^2 k \right)
\end{align*}
In summary, it holds with probability at least $1-6\exp(-k \gamma^2/72) - 2 \exp(-25 \gamma^2 k) \geq 1-7 \exp(-k \gamma^2/72)$ that
\begin{align*}
|\langle h_{\rA,\rt}(w),\rA x\rangle - \langle w, x\rangle| &\leq |\langle h_{\rA,\rt}(w),\rA x\rangle - \langle \rA w, \rA x \rangle | + |\langle \rA w, \rA x \rangle - \langle w, x\rangle|\\
&\leq |\langle h_{\rA,\rt}(w)-\rA w,\rA x\rangle| + \gamma/2\\
&< \gamma.
\end{align*}
\end{proof}

\begin{customrmk}{\ref{rmk:pIsProb}}
The value $p(z_i)$ satisfying~\eqref{eq:expectround} has $p(z_i) \in [0,1]$.
\end{customrmk}
\begin{proof} 
Recall that~\eqref{eq:expectround} states that $p(z_i)$ satisfies
\begin{align*}
(Aw)_i &= 
p(z_i)\left(\frac{1}{2 \cdot 10 \sqrt{k}} +\frac{z_i}{10\sqrt{k}}\right) + (1-p(z_i))\left(\frac{1}{2 \cdot 10 \sqrt{k}} +\frac{z_i+1}{10\sqrt{k}}\right)
\end{align*}
where $z_i$ is such that
\[
(1/2)(10 \sqrt{k})^{-1}  + z_i (10 \sqrt{k})^{-1} \leq (Aw)_i < (1/2)(10 \sqrt{k})^{-1}  + (z_i+1) (10\sqrt{k})^{-1}.
\]
This implies
\begin{align*}
    ((1/2)(10 \sqrt{k})^{-1} +z_i (10\sqrt{k})^{-1}) + (1-p(z_i))(10\sqrt{k})^{-1} &= (Aw)_i \Rightarrow \\
    (Aw)_i -((1/2)(10 \sqrt{k})^{-1} +z_i (10\sqrt{k})^{-1})&=(1-p(z_i))(10\sqrt{k})^{-1}.
\end{align*}
By definition of $z_i$, we have that the left hand side is a number in $[0,(10\sqrt{k})^{-1}]$ and thus we conclude
\[
(1-p(z_i)) \in [0,1] \Rightarrow p(z_i) \in [0,1].
\]
\end{proof}

\begin{customrmk}{\ref{rmk:phirho}}
For any training set $S$ and distribution $\cD$ over $\finalX \times \{-1,1\}$, we have
\begin{align*}
    \E_{\rA,\rt} [\Pr_{(\rx,\ry) \sim \cD}[\ry \langle h_{\rA,\rt}(w), \rA \rx\rangle > \gamma_i/2 \wedge \ry \langle w, \rx\rangle \leq 0]] &\leq\E_{(\rx,\ry) \sim \cD}[\phi(\ry \langle w, \rx \rangle)] \\
    \E_{\rA,\rt} [\Pr_{(\rx,\ry) \sim S}[\ry \langle h_{\rA,\rt}(w), \rA \rx\rangle > \gamma_i/2 \wedge \ry \langle w, \rx\rangle \leq \gamma]] &\geq \E_{(\rx,\ry) \sim S}[\phi(\ry \langle w, \rx \rangle)] \\
    \E_{\rA,\rt} [\Pr_{(\rx,\ry) \sim S}[\ry \langle h_{\rA,\rt}(w), \rA \rx\rangle \leq \gamma_i/2 \wedge \ry \langle w, \rx\rangle > \gamma]] &\leq \E_{(\rx, \ry) \sim S}[\rho(\ry \langle w, \rx \rangle)]\\
    \E_{\rA,\rt} [\Pr_{(\rx,\ry) \sim \cD}[\ry \langle h_{\rA,\rt}(w), \rA \rx\rangle \leq \gamma_i/2 \wedge \ry \langle w, \rx\rangle > 0]] &\geq \E_{(\rx, \ry) \sim \cD}[\rho(\ry\langle w, \rx \rangle)].
\end{align*}
\end{customrmk}
In the proof, we will need the following monotonicity properties 
\begin{claim}
\label{clm:monotone1}
We have $\Pr_{\rA,\rt}[y \ipr{h_{\rA,\rt}(w),\rA x} > \gamma_i/2 \mid y \ipr{w,x}=\alpha_1] \leq \Pr_{\rA,\rt}[y \ipr{h_{\rA,\rt}(w),\rA x} > \gamma_i/2 \mid y \ipr{w,x}=\alpha_2]$ for any $0 \leq \alpha_1 \leq \alpha_2 \leq \gamma_i$.
\end{claim}
\begin{claim}
\label{clm:monotone2}
We have $\Pr_{\rA,\rt}[y \ipr{h_{\rA,\rt}(w),\rA x} \leq \gamma_i/2 \mid y \ipr{w,x}=\alpha_2] \leq \Pr_{\rA,\rt}[y \ipr{h_{\rA,\rt}(w),\rA x} \leq \gamma_i/2 \mid y \ipr{w,x}=\alpha_1]$ for any $0 < \alpha_1 \leq \alpha_2 \leq \gamma_i$.
\end{claim}
First we will prove Remark~\ref{rmk:phirho} using the two claims. Afterward, we will prove Claim~\ref{clm:monotone1} and Claim~\ref{clm:monotone2}.

\begin{proof}[Proof of Remark~\ref{rmk:phirho}]
For convenience, let us recall the definitions of $\phi$ and $\rho$:
\[
\phi(\alpha) = \begin{cases} \Pr_{\rA, \rt}[y \langle h_{\rA,\rt}(w), \rA x\rangle > \gamma_i/2 \mid y \langle w, x \rangle = \alpha] & \text{if } -c_\gamma \leq \alpha \leq 0 \\
                      \frac{(\gamma_i-\alpha)}{\gamma_i}\Pr_{\rA, \rt}[y \langle h_{\rA,\rt}(w), \rA x\rangle > \gamma_i/2 \mid y \langle w, x \rangle = 0]                                    & \text{if } 0 < \alpha \leq \gamma_i      \\
                      0 & \text{if } \gamma_i < \alpha \leq c_\gamma
        \end{cases}
\]
\[
\rho(\alpha) = \begin{cases} \Pr_{\rA, \rt}[y \langle h_{\rA,\rt}(w), \rA x\rangle \leq \gamma_i/2 \mid y \langle w, x \rangle = \alpha] & \text{if } \gamma_i < \alpha \leq c_\gamma \\
                      \frac{\alpha}{\gamma_i}\Pr_{\rA, \rt}[y \langle h_{\rA,\rt}(w), \rA x\rangle \leq \gamma_i/2 \mid y \langle w, x \rangle = \gamma_i]                                    & \text{if } 0 < \alpha \leq \gamma_i      \\
                      0 & \text{if } -c_\gamma \leq \alpha \leq 0
        \end{cases}
\]
We handle each of the inequalities in turn. First we see that
\begin{align*}
\E_{\rA,\rt} [\Pr_{(\rx,\ry) \sim \cD}[\ry \langle h_{\rA,\rt}(w), \rA \rx\rangle > \gamma_i/2 \wedge \ry \langle w, \rx\rangle \leq 0] &= \\
\E_{(\rx,\ry) \sim \cD}[\Pr_{\rA,\rt} [\ry \langle h_{\rA,\rt}(w), \rA \rx\rangle > \gamma_i/2 \wedge \ry \langle w, \rx\rangle \leq 0]] 
&\leq \\
\E_{(\rx,\ry) \sim \cD}[\phi(\ry \langle w, \rx \rangle)].
\end{align*}
Here the inequality follows from the observations that $\phi(y \langle w, x \rangle) \geq 0$ for $y \langle w, x \rangle > 0$, whereas $\Pr_{\rA,\rt} [\ry \langle h_{\rA,\rt}(w), \rA \rx\rangle > \gamma_i/2 \wedge \ry \langle w, \rx\rangle \leq 0] = 0$ for such $y \langle w, x \rangle$. Similarly for $y\langle w, x \rangle = \alpha \leq 0$, we have $\phi(y \langle w, x \rangle) = \Pr_{\rA,\rt}[y\ipr{h_{\rA,\rt}(w),\rA x} > \gamma_i/2 \mid y\langle w, x \rangle = \alpha] = \Pr_{\rA,\rt}[y\ipr{h_{\rA,\rt}(w),\rA x} > \gamma_i/2 \wedge y \ipr{w,x} \leq 0 \mid y\langle w, x \rangle = \alpha]$.

Similarly, we have
\begin{align*}
\E_{\rA,\rt} [\Pr_{(\rx,\ry) \sim S}[\ry \langle h_{\rA,\rt}(w), \rA \rx\rangle > \gamma_i/2 \wedge \ry \langle w, \rx\rangle \leq \gamma] &= \\
\E_{(\rx,\ry) \sim S}[\Pr_{\rA,\rt} [\ry \langle h_{\rA,\rt}(w), \rA x\rangle > \gamma_i/2 \wedge \ry \langle w, \rx\rangle \leq \gamma]] &\geq \\
\E_{(\rx,\ry) \sim S}[\Pr_{\rA,\rt} [\ry \langle h_{\rA,\rt}(w), \rA \rx\rangle > \gamma_i/2 \wedge \ry \langle w, \rx\rangle \leq \gamma_{i}]] &\geq \\
\E_{(\rx,\ry) \sim S}[\phi(\ry \langle w, \rx \rangle)].
\end{align*}
The last inequality follows by observing that if $y \langle w, x \rangle > \gamma_i$, we have $\phi(y \ipr{w,x})=0$ and $\Pr_{\rA,\rt}[y \ipr{h_{\rA,\rt}(w),\rA x} > \gamma_i/2 \wedge y \ipr{w,x} \leq \gamma_i]=0$. For $\alpha = y \ipr{w,x}$ with $0 < \alpha \leq \gamma_i$, we have $\phi(\alpha)=\tfrac{\gamma_i-\alpha}{\gamma_i}\Pr_{\rA,\rt}[y\ipr{h_{\rA,\rt}(w),\rA x } > \gamma_i/2 \mid y \ipr{w,x}=0] \leq \Pr_{\rA,\rt}[y\ipr{h_{\rA,\rt}(w),\rA x } > \gamma_i/2 \mid y \ipr{w,x}=\alpha] = \Pr_{\rA,\rt}[y\ipr{h_{\rA,\rt}(w),\rA x } > \gamma_i/2 \wedge y\ipr{w,x}\leq \gamma_i \mid y \ipr{w,x}=\alpha]$. This uses the monotonicity in Claim~\ref{clm:monotone1}. Finally for $y\ipr{w,x} = \alpha \leq 0$, the two coincide as in the above argument.

Symmetric arguments for $\rho$ gives
\begin{align*}
\E_{\rA,\rt} [\Pr_{(\rx,\ry) \sim S}[\ry \langle h_{\rA,\rt}(w), \rA \rx\rangle \leq \gamma_i/2 \wedge \ry \langle w, \rx\rangle > \gamma] &= \\
\E_{(\rx,\ry) \sim S} [\Pr_{\rA,\rt}[\ry \langle h_{\rA,\rt}(w), \rA \rx\rangle \leq \gamma_i/2 \wedge \ry \langle w, \rx\rangle > \gamma] &\leq \\
\E_{(\rx,\ry) \sim S} [\Pr_{\rA,\rt}[\ry \langle h_{\rA,\rt}(w), \rA \rx\rangle \leq \gamma_i/2 \wedge \ry \langle w, \rx\rangle > \gamma_i] &\leq \\
\E_{(\rx, \ry) \sim S}[\rho(\ry \langle w, \rx \rangle)].
\end{align*}
Here the last inequality follows from the following considerations. For $y \ipr{w,x}=\alpha$ with $\alpha \leq \gamma_i$, we have that $\Pr_{\rA,\rt}[y \langle h_{\rA,\rt}(w), \rA x\rangle \leq \gamma_i/2 \wedge y \langle w, x\rangle > \gamma_i] = 0$ and $\rho$ is always non-negative. For $\alpha > \gamma_i$, we have $\rho(\alpha) = \Pr_{\rA,\rt}[y \langle h_{\rA,\rt}(w), \rA x\rangle \leq \gamma_i/2 \mid y \langle w, x\rangle = \alpha] = \Pr_{\rA,\rt}[y \langle h_{\rA,\rt}(w), \rA x\rangle \leq \gamma_i/2 \wedge y \ipr{w,x} > \gamma_i \mid y \langle w, x\rangle = \alpha]$ and the two coincide.

Finally, we have
\begin{align*}
\E_{\rA,\rt} [\Pr_{(\rx,\ry) \sim \cD}[\ry \langle h_{\rA,\rt}(w), \rA \rx\rangle \leq \gamma_i/2 \wedge \ry \langle w, \rx\rangle > 0] &= \\
\E_{(\rx,\ry) \sim \cD} [\Pr_{\rA,\rt}[\ry \langle h_{\rA,\rt}(w), \rA \rx\rangle \leq \gamma_i/2 \wedge \ry \langle w, \rx\rangle > 0] &\geq \\
\E_{(\rx, \ry) \sim \cD}[\rho(\ry\langle w, \rx \rangle)].
\end{align*}
Here the inequality follows by observing that for $y\ipr{w,x} = \alpha$ with $\alpha \leq 0$, both $\rho(\alpha)$ and $\Pr_{\rA,\rt}[y \langle h_{\rA,\rt}(w), \rA x\rangle \leq \gamma_i/2 \wedge y \langle w, x\rangle > 0]$ are $0$. For $0 \leq \alpha \leq \gamma_i$ we have by definition that $\rho(\alpha) = \tfrac{\alpha}{\gamma_i} \Pr_{\rA,\rt}[y\ipr{h_{\rA,\rt}(w),\rA x} \leq \gamma_i/2 \mid y\ipr{w,x} = \gamma_i] \leq \Pr_{\rA,\rt}[y\ipr{h_{\rA,\rt}(w),\rA x} \leq \gamma_i/2 \mid y\ipr{w,x} = \alpha] = \Pr_{\rA,\rt}[y\ipr{h_{\rA,\rt}(w),\rA x} \leq \gamma_i/2 \wedge y \ipr{w,x}>0 \mid y\ipr{w,x} = \alpha]$, where we used that $\Pr_{\rA,\rt}[y\ipr{h_{\rA,\rt}(w),\rA x} \leq \gamma_i/2 \mid y\ipr{w,x} = \alpha]$ is decreasing in $\alpha$ (as stated in Claim~\ref{clm:monotone2}). Finally, for $\alpha > \gamma_i$, the two coincide as above.
\end{proof}

\begin{proof}[Proof of Claim~\ref{clm:monotone1}]
Let $w_1,x_1,y_1$ be such that $\alpha_1 := y_1\langle w_1,x_1\rangle$ and let $w_2,x_2,y_2$ be such that $\alpha_2 := y_2 \ipr{w_2,x_2}$. Consider sampling $\rX_i,\rY_i \sim \Norm(0,1/k)$ independently. Also sample offsets $\rt_1',\dots,\rt_k'$ uniformly and independently in $[0,1]$ and let $\rX'_i$ be $\rX_i$ rounded based on $\rt_i'$ as above. Let $\rZ_1 = \rY =\alpha_1 \rX + \sqrt{1-\alpha_1^2}\rY$ and $\rZ_2 = \alpha_2 \rX + \sqrt{1-\alpha_2^2}\rY$. Then the marginal distribution of $\langle \rX', \rZ_j\rangle$ equals the distribution of $\ipr{h_{\rA,\rt}(w_j), y_j \rA x_j} = y_j\ipr{h_{\rA,\rt}(w_j), \rA x_j} $. 

Consider now an arbitrary outcome $X',X$ of $\rX, \rX'$. We have $\langle \rZ_j , X' \rangle \geq \gamma_i/2$ if and only if $\alpha_j \langle X, X'\rangle + \sqrt{1-\alpha_j^2} \langle \rY, X'\rangle \geq \gamma_i/2$. We also have that $\langle \rY, X'\rangle \sim \Norm(0,\|X'\|_2^2/k)$ and thus
\begin{align}
\Pr[\ipr{\rZ_2,X'}\geq \gamma_i/2] - \Pr[\ipr{\rZ_1,X'}\geq \gamma_i/2] &= \nonumber\\
\left(1-\Phi\left(\sqrt{k} \cdot \frac{\gamma_i/2-\alpha_2\ipr{X,X'}}{\sqrt{1-\alpha_2^2}}\right)\right)-\left(1-\Phi\left(\sqrt{k} \cdot \frac{\gamma_i/2-\alpha_1\ipr{X,X'}}{\sqrt{1-\alpha_1^2}}\right)\right)&= \nonumber\\
\Phi\left(\sqrt{k} \cdot \frac{\gamma_i/2-\alpha_1\ipr{X,X'}}{\sqrt{1-\alpha_1^2}}\right)-\Phi\left(\sqrt{k} \cdot \frac{\gamma_i/2-\alpha_2\ipr{X,X'}}{\sqrt{1-\alpha_2^2}}\right). \label{eq:diffcdf}
\end{align}
Here $\Phi(\cdot)$ denotes the cumulative density function of the normal distribution with mean $0$ and variance $1$. Now let
\[
u := \sqrt{k} \cdot \frac{\gamma_i/2-\alpha_1\ipr{X,X'}}{\sqrt{1-\alpha_1^2}}.
\]
and
\[
\ell := \sqrt{k} \cdot \frac{\gamma_i/2-\alpha_2\ipr{X,X'}}{\sqrt{1-\alpha_2^2}}
\]
Consider now the derivative
\begin{align*}
    \frac{\partial}{\partial \alpha} \sqrt{k} \cdot \frac{\gamma_i/2-\alpha\ipr{X,X'}}{\sqrt{1-\alpha^2}} &= \sqrt{k} \cdot
    \frac{\alpha \gamma_i/2 - \ipr{X,X'}}{(1-\alpha^2)^{3/2}}.
\end{align*}
Assume first that $\|X\|_2^2 \geq 9/10$. Then $\ipr{X,X'} \geq 8/9$ by Remark~\ref{rmk:ipXX'}. Now since $\alpha \gamma_i/2 \leq \gamma_i^2/2 \leq c_\gamma^2/8 \leq 1/9$ for $c_\gamma$ small enough. Thus the derivative when $\|X\|_2^2 \geq 9/10$ is no more than
\begin{align*}
    \sqrt{k} \cdot (1/9 - 8/9) \leq -7\sqrt{k}/9.
\end{align*}
This implies $u - \ell \geq 7(\alpha_2-\alpha_1) \sqrt{k}/9 > 0$ and therefore 
\[
\Pr[\ipr{\rZ_2,X'}\geq \gamma_i/2] \geq \Pr[\ipr{\rZ_1,X'}\geq \gamma_i/2].
\]
If we in addition have that $\|X\|_2^2 \leq 4/3$, then we may even show that the difference in probabilities is large as a function of $\alpha_2-\alpha_1$ as follows
\begin{align*}
    \Pr[\ipr{\rZ_2,X'}\geq \gamma_i/2] - \Pr[\ipr{\rZ_1,X'}\geq \gamma_i/2] &=\frac{1}{\sqrt{2 \pi}} \cdot \int_{x=\ell}^u e^{-x^2/2} dx \\
    &\geq e^{-\max_{a \in [\ell ,u]} a^2/2} \frac{7 \sqrt{k} (\alpha_2-\alpha_1)}{9 \sqrt{2 \pi}}.
\end{align*}
Observing that
\[
\max_{a \in [\ell, u]} a^2 \leq \frac{k}{1-c_\gamma^2} \cdot \max\{\gamma^2_i/2, \gamma_i^2 \ipr{X,X'}^2\}
\]
we use Remark~\ref{rmk:ipXX'} to conclude $\ipr{X,X'} \leq (10/9)\|X\|_2^2$ and thus $u^2 \leq 2k \gamma_i^2 (10/9)^2 \leq 3k \gamma_i^2$ for $c_\gamma \leq 1/\sqrt{2}$. This gives us that for any $X$ with $9/10 \leq \|X\|_2^2 \leq 4/3$, it holds that
\begin{align*}
    \Pr[\ipr{\rZ_2,X'}\geq \gamma_i/2] - \Pr[\ipr{\rZ_1,X'}\geq \gamma_i/2] &\geq e^{-3k \gamma_i^2/2} \frac{7 \sqrt{k} (\alpha_2-\alpha_1)}{9 \sqrt{2 \pi}}.
\end{align*}
For $\|X\|_2^2 < 9/10$, we have $\|X'\|_2 = \|X'-X + X\|_2 \leq \|X'-X\|_2 + \|X\|_2 \leq \sqrt{k (10 \sqrt{k})^{-2}} + \sqrt{9/10} \leq 11/10$. It follows by Cauchy-Schwartz that $|\ipr{X,X'}| \leq \|X\|_2 \cdot \|X'\|_2 \leq \sqrt{9/10} \cdot 11/10 \leq 11/10$. For $0 \leq \alpha \leq \gamma_i \leq c_\gamma/2 \leq 1/\sqrt{8}$ for $c_\gamma \leq 1/\sqrt{2}$, this upper bounds the derivative by
\[
\sqrt{k} \cdot \frac{\gamma_i^2/2 + 11/10}{(1-1/8)^{3/2}} < 2 \sqrt{k}.
\]
If $u \geq \ell$, we already have that 
\[
\Pr[\ipr{\rZ^2,X'}\geq \gamma_i/2] - \Pr[\ipr{\rZ^1,X'}\geq \gamma_i/2] \geq 0
\]
So assume $u < \ell$. The bound on the derivative gives us that $\ell-u \leq 2 \sqrt{k} (\alpha_2-\alpha_1)$ and we conclude
\begin{align*}
    \Pr[\ipr{\rZ^2,X'}\geq \gamma_i/2] - \Pr[\ipr{\rZ^1,X'}\geq \gamma_i/2] &=-\frac{1}{\sqrt{2 \pi}} \cdot \int_{x=u}^\ell e^{-x^2/2} dx \\
    &\geq -e^{-\min_{a \in [u,\ell]} a^2/2} \cdot \frac{2 \sqrt{k} (\alpha_2-\alpha_1)}{\sqrt{2 \pi}} \\
    &\geq -\frac{2 \sqrt{k} (\alpha_2-\alpha_1)}{\sqrt{2 \pi}}.
\end{align*}
We finally conclude
\begin{align}
    \Pr[\ipr{\rZ^2,X'}\geq \gamma_i/2] - \Pr[\ipr{\rZ^1,X'}\geq \gamma_i/2] &\geq \nonumber\\
    \Pr[9/10 \leq \|\rX\|_2^2 \leq 4/3] \cdot e^{- 3k\gamma_i^2/2} \cdot \frac{7 \sqrt{k}(\alpha_2-\alpha_1)}{9 \sqrt{2 \pi}} - \Pr[\|\rX\|_2^2 < 9/10] \cdot \frac{2 \sqrt{k} (\alpha_2-\alpha_1)}{\sqrt{2 \pi}}.\label{eq:propdiff}
\end{align}
Using Theorem~\ref{thm:chibound}, we get
\begin{align}
\Pr[9/10 \leq \|\rX\|_2^2 \leq 4/3] \geq 1-2\exp(-k/800), \label{eq:likelyCase}
\end{align}
and
\[
\Pr[\|\rX\|_2^2 < 9/10] \leq 2\exp(-k/800).
\]
For $k$ at least a sufficiently large constant, we have that~\eqref{eq:likelyCase} is at least $1/2$ and we get that~\eqref{eq:propdiff} is at least
\begin{align*}
    e^{- 3k\gamma_i^2/2} \cdot \frac{7 \sqrt{k}(\alpha_2-\alpha_1)}{18 \sqrt{2 \pi}} - e^{-k/800} \cdot \frac{4 \sqrt{k} (\alpha_2-\alpha_1)}{\sqrt{2 \pi}}.
\end{align*}
For the constant $c_\gamma$ sufficiently small, this is positive as $\gamma_i \leq c_\gamma$.
\end{proof}

\begin{proof}[Proof of Claim~\ref{clm:monotone2}]
Similarly to the proof of Claim~\ref{clm:monotone1}, let $w_1,x_1,y_1$ by such that $\alpha_1 = y_1 \ipr{w_1,x_1}$ and let $w_2,x_2,y_2$ be such that $\alpha_2 = y_2 \ipr{w_2,x_2}$. Draw $\rX$, $\rX'$ and $\rZ_1,\rZ_2$ as above. Consider again an arbitrary outcome $X',X$ of $\rX,\rX'$. We have $\ipr{\rZ_j,X'} \leq \gamma_i/2$ if and only if $\alpha_j \ipr{X,X'} + \sqrt{1-\alpha_j^2} \ipr{\rY,X'} \leq \gamma_i/2$. Hence
\begin{align*}
    \Pr[\ipr{\rZ_1,X'} \leq \gamma_i/2] - \Pr[\ipr{\rZ_2,X'} \leq \gamma_i/2] &=\\
\Phi\left(\sqrt{k} \cdot \frac{\gamma_i/2-\alpha_1\ipr{X,X'}}{\sqrt{1-\alpha_1^2}}\right)-\Phi\left(\sqrt{k} \cdot \frac{\gamma_i/2-\alpha_2\ipr{X,X'}}{\sqrt{1-\alpha_2^2}}\right)
\end{align*}
This has the exact same constraints $0 \leq \alpha_1 \leq \alpha_2 \leq \gamma_i$ and exact same form as~\eqref{eq:diffcdf}. The conclusion thus follows from the proof of Claim~\ref{clm:monotone1}.
\end{proof}

\begin{customrmk}{\ref{rem:case_ii_norm_bound}}
   If $\|X\|_2^2 \leq 4/3$, then $\norm{X'}_2^2 < 2$.
\end{customrmk}
%the following are proofs used in subsection _  to bound
\begin{proof}
By the triangle inequality, and using that all coordinates of $X-X'$ are bounded by $(10 \sqrt{k})^{-1}$ in absolute value, we have
\begin{align*}
    \norm{X'}_2^2 &= \norm{X'-X+X}_2^2 \\
    &\leq \left(\|X'-X\|_2 + \|X\|_2\right)^2 \\
    &\leq \left(\sqrt{k (10 \sqrt{k})^{-2}} + \sqrt{4/3}\right)^2 \\
    &= (1/10 + \sqrt{4/3})^2 \\
    &< 2.
\end{align*}
\end{proof}

\begin{customrmk}{\ref{rmk:ipXX'}}
    If $\norm{X}_2^2\ge 9/10$, then $(8/9)\norm{X}_2^2\le \ipr{X,X'}\le (10/9)\norm{X}_2^2$
\end{customrmk}

\begin{proof}
We have:
\begin{align*}
    \ipr{X',X} &= \ipr{X'-X + X,X} \\
    &= \ipr{X'-X,X} + \|X\|_2^2.
\end{align*}
Since each coordinate of $X'-X$ is bounded by $(10 \sqrt{k})^{-1}$ in absolute value, it follows by Cauchy-Schwartz that
\begin{align*}
    |\ipr{X'-X,X}| &\leq \|X'-X\|_2 \cdot \|X\|_2 \\
    &\leq \sqrt{k (10 \sqrt{k})^{-2}} \cdot \frac{\|X\|^2_2}{\|X\|_2} \\
    &\leq \frac{\|X\|_2^2}{10 \sqrt{9/10}} \\
    &\leq \|X\|_2^2/9.
\end{align*}
The conclusion follows.
\end{proof}

\begin{customrmk}{\ref{rmk:concentrationx}}
For any distribution $\cD$ over $\finalX \times \{-1,1\}$, fixed $w \in \finalH$, margin $\gamma$ and any $A \in \R^{k \times d}$, it holds with probability at least $1-\delta$ over $\rS \sim \cD^n$ that
\[
|\Loss_{A\cD}^{\gamma}(w) - \Loss_{A\rS}^{\gamma}(w)| \leq \sqrt{\frac{8\Loss_{A \cD}^{\gamma}(w)\ln(1/\delta)}{n}} + \frac{2 \ln(1/\delta)}{n}.
\]
\end{customrmk}
\begin{proof}
Since $\Loss^\gamma_{A \rS}(w)$ is an average of $n$ i.i.d.\ $0/1$ random variables with mean $\Loss^\gamma_{A \cD}(w)$, we get from Bernstein's inequality that 
\begin{align*}
    \Pr_{\rS \sim \cD}\left[|\Loss_{A\cD}^{\gamma}(w) - \Loss_{A\rS}^{\gamma}(w)| >\sqrt{\frac{8\Loss_{A \cD}^{\gamma}(w)\ln(1/\delta)}{n}} + \frac{2 \ln(1/\delta)}{n} \right] &\leq \\
    \exp\left(- \frac{\tfrac{1}{2} \cdot \left(\sqrt{ 8\Loss_{A \cD}^{\gamma}(w) n \ln(1/\delta) } + 2 \ln(1/\delta)\right)^2}{n \Loss_{A \cD}^\gamma(w) + \tfrac{1}{3} \cdot \left(\sqrt{ 8\Loss_{A \cD}^{\gamma}(w)\ln(1/\delta) n} + 2 \ln(1/\delta)\right) }\right) &\leq\\
    \exp\left(- \frac{\tfrac{1}{2} \cdot \max\left\{8\Loss_{A \cD}^{\gamma}(w) n,  4 \ln(1/\delta)\right\} \ln(2/\delta)}{\tfrac{1}{8} \max\{n \Loss_{A \cD}^\gamma(w), 4\ln(1/\delta)\} + \tfrac{1}{3} \cdot \sqrt{2 \cdot \max\{ 8\Loss_{A \cD}^{\gamma}(w), 4 \ln(1/\delta)\} \cdot \ln(1/\delta)}}\right).
\end{align*}
Using that $\ln(1/\delta) \leq \tfrac{1}{4} \max\{ 8\Loss_{A \cD}^{\gamma}(w), 4 \ln(1/\delta)\}$, this is at most
\[
    \exp\left(- \frac{\tfrac{1}{2} \ln(1/\delta)}{\tfrac{1}{8}  + \tfrac{1}{3} \cdot \sqrt{\tfrac{1}{2} }}\right) \leq \exp(-\ln(1/\delta)) = \delta.
\]
\end{proof}